\documentclass[final]{colt2019}

\makeatletter
\def\set@curr@file#1{\def\@curr@file{#1}} 
\makeatother

\usepackage{times}
\usepackage{amsmath}
\usepackage{amssymb}
\usepackage{mathrsfs}
\usepackage{mathabx}
\usepackage{bbm}
\usepackage{bm}
\usepackage{multirow}
\usepackage{enumitem}
\usepackage{marginnote}
\usepackage{graphicx}
\usepackage{mdframed}
\usepackage{mathcomp}
\usepackage{stmaryrd}
\usepackage{mathtools}
\usepackage{xr}
\usepackage{nicefrac}       

\newcommand{\bigto}{\wtilde{O}}

\newcommand{\cW}{\mathcal{W}}

\newcommand{\wtilde}{\widetilde}

\newcommand{\loss}{\ell}
\DeclareBoldMathCommand{\vloss}{\loss}
\DeclareBoldMathCommand{\grad}{g}
\DeclareBoldMathCommand{\fakegrad}{\mathring{\bm{g}}}
\DeclareBoldMathCommand{\e}{e}
\DeclareBoldMathCommand{\p}{p}
\DeclareBoldMathCommand{\u}{u}
\DeclareBoldMathCommand{\w}{w}
\DeclareBoldMathCommand{\x}{x}
\DeclareBoldMathCommand{\vzero}{0}
\let\top\intercal

\newcommand{\reals}{\mathbb{R}}

\renewcommand{\x}{\bm{x}}   
\newcommand{\z}{\bm{z}}
 
\newcommand{\g}{\bm{g}}

\DeclareMathOperator*{\argmin}{argmin}

\newcommand{\G}{\mathbf{G}}

\newcommand*{\what}[1]{\widehat{#1}}

\newcommand*{\op}[1]{\operatorname{#1}}

\newcommand{\inner}[2]{\langle#1,#2\rangle}

\newcommand{\norm}[1]{\left\lVert#1\right\rVert}

\newmdtheoremenv{condition}{Condition}

\newcommand{\vertiii}[1]{{\left\vert\kern-0.25ex\left\vert\kern-0.25ex\left\vert #1 
    \right\vert\kern-0.25ex\right\vert\kern-0.25ex\right\vert}}

\makeatletter
\newcommand*\bcdot{\mathpalette\bigcdot@{.5}}
\newcommand*\bigcdot@[2]{\mathbin{\vcenter{\hbox{\scalebox{#2}{$\m@th#1\bullet$}}}}}

\newlength\myindent
\usepackage{algorithm}
\usepackage{algorithmic}
\usepackage{comment}
\usepackage{bm}

\newcommand{\freegrad}{$\textsc{FreeGrad}$}
\newcommand{\matfreegrad}{$\textsc{Matrix-FreeGrad}$}
\newcommand{\freerange}{$\textsc{FreeRange}$}

\newtheorem{assumption}{Assumption}

\usepackage{todonotes}

\usepackage{booktabs}
\usepackage{footnote}
\makesavenoteenv{tabular}
\makesavenoteenv{table}
\let\norm\undefined
\let\set\undefined

\usepackage{notation}
\usepackage{bm}
\usetikzlibrary{patterns,snakes}
\usepackage{mathtools}
\mathtoolsset{showonlyrefs}
\usepackage{multirow}
\usepackage{microtype}
\newcommand{\eps}{}

\newcommand{\bV}{\bm{V}}

\let\footnotetitle\thanks
\DeclareRobustCommand{\VAN}[3]{#2} 
\usepackage[mathscr]{euscript}

\makeatletter
\g@addto@macro\bfseries{\boldmath}
\makeatother

\begin{document}

	\title[Lipschitz-Comparator-Norm-Adaptivity]{Lipschitz and Comparator-Norm Adaptivity in Online Learning\protect\footnotetitle{Accepted for presentation at the Conference on Learning Theory (COLT) 2020}}
	
	\coltauthor{\Name{Zakaria Mhammedi} \Email{zak.mhammedi@anu.edu.au} \\
	\addr	The Australian National University and Data61 \\
		\Name{Wouter M.\ Koolen} \Email{wmkoolen@cwi.nl} \\
		\addr Centrum Wiskunde \& Informatica 
	}

	\maketitle

\begin{abstract}
  We study Online Convex Optimization in the unbounded setting where neither predictions nor gradient are constrained. The goal is to simultaneously adapt to both the sequence of gradients and the comparator. We first develop parameter-free and scale-free algorithms for a simplified setting with hints. We present two versions: the first adapts to the squared norms of both comparator and gradients separately using $O(d)$ time per round, the second adapts to their squared inner products (which measure variance only in the comparator direction) in time $O(d^3)$ per round. We then generalize two prior reductions to the unbounded setting; one to not need hints, and a second to deal with the range ratio problem (which already arises in prior work). We discuss their optimality in light of prior and new lower bounds. We apply our methods to obtain sharper regret bounds for scale-invariant online prediction with linear models.
\end{abstract}


	\section{Introduction}
	\label{sec:intro}
We consider the setting of online convex optimization where the goal is to make sequential predictions to minimize a certain notion of \emph{regret}. Specifically, at the beginning of each round $t\geq 1$, a \emph{learner} predicts $\what\w_t$ in some convex set $\cW \subseteq \reals^d$ in dimension $d\in\mathbb{N}$. The \emph{environment} then reveals a convex loss function $f_t\colon \cW \rightarrow \reals$, and the learner suffers loss $f_t(\what\w_t)$. The goal of the learner is to minimize the regret $\sum_{t=1}^T f_t(\what \w_t) -  \sum_{t=1}^T f_t(\w)$ after $T\geq 1$ rounds against any ``comparator'' prediction $\w\in \cW$. Typically, an online learning algorithm outputs a vector $\what\w_t$, $t\geq 1$, based only on a sequence of observed sub-gradients $(\g_s)_{s< t}$, where $\g_s\in \partial f_s(\what\w_s), s<t$. In this paper, we are interested in online algorithms which can guarantee a good regret bound (by a measure which we will make precise below) against any comparator vector $\w\in \cW$, even when $\cW$ is unbounded, and without prior knowledge of the maximum norm $L\coloneqq\max_{t\leq T}\|\g_t\|$ of the observed sub-gradients. In what follows, we refer to $L$ as the \emph{Lipschitz constant}.

By assuming an upper-bound $D>0$ on the norm of the desired comparator vector $\w$ in hindsight, there exist \emph{Lipschitz-adaptive} algorithms that can achieve a sub-linear regret of order $L D \sqrt{T}$, without knowing $L$ in advance. A Lipschitz-adaptive algorithm is also called \emph{scale-free} (or scale-invariant) if its predictions do not change when the loss functions $(f_t)$ are multiplied by a factor $c>0$; in this case, its regret bound is expected to scale by the same factor $c$. When $L$ is known in advance and $\cW=\reals^d$, there exists another type of algorithms, so-called \emph{parameter-free}, which can achieve an $\bigto(\|\w\| L \sqrt{T})$ regret bound, where $\w$ is the desired comparator vector in hindsight (the notation $\bigto$ hides log-factors). Up to an additive lower-order term, this type of regret bound is also achievable for bounded $\cW$ via the unconstrained-to-constrained reduction \citep{cutkosky2019artificial}.

The question of whether an algorithm can simultaneously be \emph{scale-free} and \emph{parameter-free} was posed as an open problem by \cite{orabona2016open}. It was latter answered in the negative by \cite{cutkosky2017online}. Nevertheless, \cite{cutkosky2019artificial} recently presented algorithms which achieve an $\bigto(\|\w\| L\sqrt{T}+ L \|\w\|^3)$ regret bound, without knowing either $L$ or $\norm{\w}$. This does not violate the earlier lower bound of \cite{cutkosky2017online}, which insists on norm dependence $\wtilde O(\norm{\w})$.
 
 Though \cite{cutkosky2019artificial} designs algorithms that can to some extent adapt to both $L$ and $\norm{\w}$, their algorithms are still \emph{not} scale-free. Multiplying $(f_t)$, and as a result $(\g_t)$, by a positive factor $c>0$ changes the outputs $(\what\w_t)$ of their algorithms, and their regret bounds scale by a factor $c'$, not necessarily equal to $c$. Their algorithms depend on a parameter $\epsilon>0$ which has to be specified in advance. This parameter appears in their regret bounds as an additive term and also in a logarithmic term of the form $\log (L^{\alpha} /\epsilon)$, for some $\alpha>1$. As a result of this type of dependence on $\epsilon$ and the fact that $\alpha>1$, there is no prior choice of $\epsilon$ which can make their regret bounds scale-invariant. What is more, without knowing $L$, there is also no ``safe'' choice of $\epsilon$ which can prevent the $\log (L^{\alpha} /\epsilon)$ term from becoming arbitrarily large relative to $L$ (it suffices for $\epsilon$ to be small enough relative to the ``unknown'' $L$).
 
\paragraph{Contributions.} Our main contribution is a new scale-free, parameter-free learning algorithm for OCO with regret at most $O(\|\w\| \sqrt{V_T\log (\|\w\|T)})$, for any comparator $\w \in \cW$ in a bounded set $\cW$, where $V_T \coloneq \sum_{t=1}^T \|\g_t\|^2$. When the set $\cW$ is unbounded, the algorithm achieves the same guarantee up to an additive $O(L \sqrt{\max_{t\leq T}B_t}+ L \|\w\|^3)$, where $B_t\coloneqq  \sum_{s=1}^t\| \g_s \|/L_t$ and $L_t \df \max_{s \le t} \norm{\g_s}$, for all $t\in[T]$. In the latter case, we also show a matching lower bound; when $\cW$ is unbounded and without knowing $L$, any online learning algorithm which insists on an $\bigto(\sqrt{T})$ bound, has regret at least $\Omega (L \sqrt{B_T}+ L \|\w\|^3)$. We also provide a second scale-invariant algorithm which replaces the leading $\|\w\| \sqrt{V_T}$ term in the regret bound of our first algorithm by $\sqrt{\w^\top \bV_T \w \ln \det \bm V_T}$, where $\bV_T \coloneqq \sum_{t=1}^T\g_t \g_t^\top$. Our starting point for designing our algorithms is a known potential function which we show to be controlled for a unique choice of output sequence $(\what\w_t)$.

As our main application, we show how our algorithms can be applied to learn linear models. The result is an online algorithm for learning linear models whose label predictions are invariant to coordinate-wise scaling of the input feature vectors. The regret bound of the algorithm is naturally also scale-invariant and improves on the bounds of existing state-of-the-art algorithms in this setting \citep{Kotlowski17,KempkaKW19}.

\paragraph{Related Work}
For an overview of Online Convex Optimization in the bounded setting, we refer to the textbook \citep{HazanOCOBook2016}. The unconstrained case was first studied by \cite{McMahanStreeter2010}. A powerful methodology for the  unbounded case is Coin Betting by \cite{orabona2016coin}. Even though not always visible, our potential functions are inspired by this style of thinking. We build our unbounded OCO learner by targeting a specific other constrained problem. We further employ several general reductions from the literature, including gradient clipping \cite{cutkosky2019artificial}, the constrained-to-unconstrained reduction \cite{cutkosky2018}, and the restart wrapper to pacify the final-vs-initial scale ratio appearing inside logarithms by \cite{mhammedi19}. Our analysis is, at its core, proving a certain minimax result about sufficient-statistic-based potentials reminiscent of the Burkholder approach pioneered by \cite{FosterRakhlinSridharan2017,Foster2018}. Applications for scale-invariant learning in linear models were studied by \cite{KempkaKW19}. For our multidimensional learner we took inspiration from the Gaussian Exp-concavity step in the analysis of the MetaGrad algorithm by \cite{Erven2016}.

\paragraph{Outline} In Section~\ref{sec:prelim}, we present the setting and notation, and formulate our goal. In Section~\ref{sec:mainalg}, we present our main algorithms. In Section~\ref{sec:lower}, we present new lower-bounds for algorithms which adapt to both the Lipschitz constant and the norm of the comparator. In Section~\ref{sec:linear}, we apply our algorithms to online prediction with linear models.

\section{Preliminaries}\label{sec:prelim}
Our goal is to design scale-free algorithms that adapt to the Lipschitz constant $L$ and comparator norm $\norm{\w}$. We will first introduce the setting, then discuss existing reductions, and finally state what needs to be done to achieve our goal.

\subsection{Setting and Notation}
Let $\cW \subseteq \reals^d, d\in \mathbb{N},$ be a convex set, and assume without loss of generality that $\bm{0}\in \cW$. We allow the set $\cW$ to be unbounded, and we define its (possibly infinite) diameter $D\coloneqq\sup_{\w,\w' \in \cW} \|\w -\w'\|\in[0,+\infty]$. We consider the setting of Online Convex Optimization (OCO) where at the beginning of each round $t\geq 1$, the learner outputs a prediction $\what\w_t \in \cW$, before observing a convex loss function $f_t: \cW \rightarrow  \reals$, or an element of its sub-gradient $\g_t \in \partial f_t(\what \w_t)$ at $\what\w_t$. The goal of the learner is to minimize the regret after $T\geq 0$ rounds, which is given by
\begin{align}
\sum_{t=1}^T f_t(\what\w_t) -     \sum_{t=1}^T f_t(\w) \quad\label{eq:linupper}
\end{align}
for any comparator vector $\w\in \cW$. In this paper, we do not assume that $T$ is known to the learner, and so we are after algorithms with so called \emph{any-time} guarantees. By convexity, we have 
\begin{align}
\sum_{t=1}^T f_t(\what\w_t) -     \sum_{t=1}^T f_t(\w) \leq \sum_{t=1}^T \inner{\g_t}{\what\w_t -\w}, \quad \text{for all $\w\in \cW$}, \label{eq:lin}
\end{align}
and thus for the purpose of minimizing the regret, typical OCO algorithms minimize the RHS of~\eqref{eq:lin}, which is known as the  \emph{linearized regret}, by generating outputs $(\what\w_t)$ based on the sequence of observed sub-gradients $(\g_t)$. Likewise, we focus our attention exclusively on linear optimization.

Given a sequence of sub-gradients $(\grad_t)$, it will be useful to define the running maximum gradient norm and the clipped sub-gradients
\begin{align} \label{eq:Lip} L_t \coloneqq \max_{s\in [t]} \| \g_s\| \quad \text{and} \quad \bar\g_t\coloneq \g_t \cdot  L_{t-1}/L_t, \end{align}
for $t\geq 1$, with the convention that $L_0=0$. We also drop the subscript $t$ from $L_t$ when $t=T$, \emph{i.e.} we write $L$ for $L_T$.

We denote by $\text{A}(\g_1,\dots, \g_{t-1};h_t)$ the output in round $t \ge 1$ of an algorithm \textsc{A}, which uses the observed sub-gradients so far and a \emph{hint} $h_t \ge L_t$ on the upcoming sub-gradient $\g_t$.
As per Section~\ref{sec:intro}, we say that an algorithm is \emph{scale-free} (or scale-invariant) if its predictions are invariant to any common positive scaling of the loss functions $(f_t)$ and, if applicable, the hints.

\paragraph{Additional Notation.} Given a closed convex set $\mathcal{X} \subseteq \reals^d$, we denote by $\Pi_{\mathcal{X}}(\bm x)$ the Euclidean projection of a point $\bm x \in \reals^d$ on the set $\mathcal X$; that is, $\Pi_{\mathcal{X}}(\bm x) \in \argmin_{\tilde{\bm{x}} \in \mathcal X} \| \bm{x} -\tilde{\bm{x}}\|$.

	\subsection{Helpful Reductions}\label{sec:helpful}
The difficulty behind designing scale-free algorithms lies partially in the fact that $L_{t}$ is unknown at the start of round $t$; before outputting $\what\w_{t}$. The following result due to \cite{cutkosky2019artificial} quantifies the additional cost of proceeding with the plug-in estimate $L_{t-1}$ for $L_t$:
	\begin{lemma}
		\label{lem:reduc1}
		Let $\textsc{A}$ be an online algorithm which at the start of each round $t\geq 1$, has access to a hint $h_t\geq L_t$, and outputs $\textsc{A}(\g_1,\dots,\g_{t-1};h_t)\in \cW$, before observing $\g_t$. Suppose that $\textsc{A}$ guarantees an upper-bound $R^{\textsc{A}}_T(\w)$ on its linearized regret for the sequence $(\g_t)$ and for all $\w\in \cW, T\geq 1$. Then, algorithm $\textsc{B}$ which at the start of each round $t\geq 1$ outputs $\what\w_t = \textsc{A}(\bar{\g}_1,\dots, \bar{\g}_{t-1};L_{t-1})$, guarantees \begin{align} \sum_{t=1}^T \inner{\what\w_t-\w}{\g_t}  \leq  R^{\textsc{A}}_T(\w) + \max_{t\in[T]} \|\what\w_t\| L_t  +\| \w \| L , \quad \text{$\forall \w\in \cW, T\geq 1.$}  \label{eq:boundtransfer}  \end{align}
	\end{lemma} 
First, we note that Lemma \ref{lem:reduc1} is only really useful when $\cW$ is bounded; otherwise, depending on algorithm $\textsc{A}$, the term $\max_{t\in[T]}L_t \|\what\w_t\|$ on the RHS of~\eqref{eq:boundtransfer} could in principle be arbitrarily large even for fixed $\w, L$, and $T$. The moral of Lemma~\ref{lem:reduc1} is that as long as the set $\cW$ is bounded, one does not really need to know $L_t$ before outputting $\what\w_t$ to guarantee a ``good'' regret bound against any $\w \in \cW$. For example, suppose that $\cW$ has a bounded diameter $D$ and algorithm \textsc{A} in Lemma \ref{lem:reduc1} is such that ${R}^{\textsc{A}}_T(\w) = \wtilde{O}(\|\w\| L \sqrt{T} + D L)$, for all $\w\in \cW$. Then, from~\eqref{eq:boundtransfer} and the fact that $\|\what\w_t\|\leq D$ (since $\what\w_t\in \cW$), it is clear that algorithm $\textsc{B}$ in Lemma \ref{lem:reduc1} also guarantees the same regret bound $R_T^{\textsc{A}}(\w)$ up to an additive $2 DL$, despite not having had the hints $(h_t)$.
	
	It is possible to extend the result of Lemma \ref{lem:reduc1} so that the regret bound of algorithm \textsc{B} remains useful even in the case where $\cW$ is unbounded. An approach suggested by \cite{cutkosky2019artificial} is to restrict the outputs $(\what\w_t)$ of algorithm \textsc{B} to be in a non-decreasing sequence $(\cW_t)$ of \emph{bounded} convex subsets of $\cW$. In this case, the diameters $(D_t) \subset \reals$ of $(\cW_{t})$ need to be carefully chosen to achieve a desired regret bound. This approach, which essentially combines the idea of Lemma \ref{lem:reduc1} and the unconstrained-to-constrained reduction due to \cite{cutkosky2018}, is formalized in the next lemma (essentially due to \cite{cutkosky2019artificial}):
	\begin{lemma}
	\label{lem:reduc2}
	Let algorithm $\textsc{A}$ be as in Lemma \ref{lem:reduc1}, and let $(\cW_t)$ be a sequence of non-decreasing closed convex subsets of $\cW$ with diameters $(D_t)\subset \reals_{>0}$. Then, algorithm $\textsc{B}$ which at the start of round $t\geq 1$ outputs $\what\w_t = \Pi_{\cW_t}(\wtilde{\w}_t)$, where \begin{gather}  \wtilde{\w}_t\coloneqq \textsc{A}(\wtilde{\g}_1,\dots, \wtilde{\g}_{t-1};L_{t-1}) \  \ \ \text{and} \ \ \ \wtilde{\g}_s \coloneqq (\bar \g_s + \|\bar \g_s\|  \cdot (\wtilde{\w}_s -\what{\w}_s)   /\| \wtilde{\w}_s -\what{\w}_s\|)/2, \ \ s< t, \shortintertext{guarantees, for all $\w\in \cW$ and $T\geq 1,$}  \sum_{t=1}^T \inner{\what\w_t-\w}{\g_t} \leq R^{\textsc{A}}_T(\w)   + \sum_{t=1}^T \|\g_t\| \cdot \|\w -\Pi_{\cW_t}(\w)\|+ L  D_T +L \| \w\|.\label{eq:boundtransfer2}  \end{gather}
\end{lemma} 
We see that compared to Lemma \ref{lem:reduc1}, the additional penalty that algorithm \textsc{B} incurs for restricting its predictions to the sets $\cW_1,\dots, \cW_T \subseteq \cW$ is the sum $\sum_{t=1}^T \|\g_t\| \cdot \|\w -\Pi_{\cW_t}(\w)\|$. The challenge is now in choosing the diameters $(D_t)$ to control the trade-off between this sum and the term $L D_T$ on the RHS of~\eqref{eq:boundtransfer2}. If $T$ is known in advance, one could set $D_1=\dots=D_T=\sqrt{T}$, in which case the RHS of~\eqref{eq:boundtransfer2} is at most \begin{align}R_T^{\textsc{A}}(\w) +  L(\|\w\|^3+\|\w\|) + L\sqrt{T}. \label{eq:naive} \end{align}
We now instantiate the bound of Lemma \ref{lem:reduc2} for another choice of $(D_t)$ when $T$ is unknown:
\begin{corollary}
	\label{cor:instan}
In  the setting of Lemma \ref{lem:reduc2}, let $\cW_t$ be the ball of diameter $D_t \coloneqq \sqrt{\max_{s\leq t} B_s}$, $t\geq 1$, where $B_t\coloneqq \sum_{s=1}^t \|\g_s\|/L_t$, and let $\cW=\reals^d$. Then the RHS of~\eqref{eq:boundtransfer2} is bounded from above by
\begin{align}
R_T^{\textsc{A}}(\w)+    L \|\w\|^3 +  L \sqrt{\max_{t\in[T]} B_t} +L \|\w\|, \quad \forall \w \in \cW =\reals^d, T\geq 1.  \label{eq:neq}
\end{align}
\end{corollary}
We see that by the more careful choice of $(D_t)$ in Corollary \ref{cor:instan}, one can replace the $L \sqrt{T}$ term in \eqref{eq:naive} by the smaller quantity $L \sqrt{\max_{t\in[T]} B_t}$; whether this can be improved further to $\sqrt{V_T}$, where $V_T = \sum_{t=1}^T \|\g_t\|^{2}$, was raised as an open question by \cite{cutkosky2019artificial}. We will answer this in the negative in Theorem \ref{thm:lower1}. We will also show in Theorem \ref{thm:lower2} below that, if one insists on a regret of order $\wtilde{O}(\sqrt{T})$, it is essentially not possible to improve on the penalty $L \|\w\|^3$ in \eqref{eq:neq}.

\subsection{Outlook}
The conclusion that should be drawn from Lemmas \ref{lem:reduc1} and \ref{lem:reduc2} is the following; if one seeks an algorithm \textsc{B} with a regret bound of the form $\wtilde O(\|\w\| L \sqrt{T})$ up to some lower-order terms in $T$, without knowledge of $L$ and regardless of whether $\cW$ is bounded or not, it suffices to find an algorithm \textsc{A} which guarantees the sought type of regret whenever it has access to a sequence of hints $(h_t)$ satisfying  (as in Lemmas \ref{lem:reduc1} and \ref{lem:reduc2}), $h_t \geq L_t$, for all $t\geq 1$. Thus, our first goal in the next section is to design a \emph{scale-free} algorithm $\textsc{A}$ which accesses such a sequence of hints and ensures that its linearized regret is bounded from above by:
\begin{align}
O\left(\|\w\|  \sqrt{V_T  \ln (\|\w\| V_T)}\right),  \ \ \ \text{where} \ \ V_T \coloneqq h_1^2+ \sum_{t=1}^T \|\g_t\|^2, \label{eq:freegrad0}
\end{align}
for all $\w  \in \reals^d, T\geq 0$, and $(\g_t)\subset\reals^d$. We show an analogous ``full-matrix'' upgrade of order $\sqrt{\w^\top \bm V \w \ln \del*{\w^\top \bm V \w \det \bm V}}$, with $\bm V = \sum_{t=1}^T \g_t \g_t^\top$. We note that if Algorithm \textsc{A} in Lemmas \ref{lem:reduc1} and \ref{lem:reduc2} is scale-free, then so is the corresponding Algorithm  \textsc{B}. 

If the desired set $\cW$ has bounded diameter $D>0$, then using the unconstrained-to-constrained reduction due to \cite{cutkosky2018}, it is straightforward to design a new algorithm based on $\textsc{A}$ with regret also bounded by \eqref{eq:freegrad0} up to an additive $L D$, for $\w\in \cW$ (this is useful for Lemma \ref{lem:reduc1}).

Finally, we also note that algorithms which can access hints $(h_t)$ such that $h_t \geq L_t$, for all $t\geq 1$,
are of independent interest; in fact, it is the same algorithm $\textsc{A}$ that we will use in Section \ref{sec:linear} as a scale-invariant algorithm for learning linear models.

\section{Scale-Free, Parameter-Free Algorithms for OCO}
\label{sec:mainalg}
In light of the conclusions of Section \ref{sec:prelim}, we will design new unconstrained scale-free algorithms which can access a sequence of hints $(h_t)$ (as in Lemma \ref{lem:reduc1}) and guarantee a regret bound of the form given in \eqref{eq:freegrad0}.
In this section, we will make the following assumption on the hints $(h_t)$:
\begin{assumption}
	\label{assum:assum2}
We assume that (i) $(h_t)$ is a non-decreasing sequence; (ii) $h_t \geq L_t$, for all $t\geq 1$; and (iii) if the sub-gradients $(\g_s)$ are multiplied by a factor $c>0$, then the hints $(h_t)$ are multiplied by the same factor $c$.
\end{assumption}
The third item of the assumption ensures that our algorithms are scale-free. We note that Assumption \ref{assum:assum2} is satisfied by the sequence of hints that Algorithm \textsc{B} constructs when invoking Algorithm \textsc{A} in Lemmas \ref{lem:reduc1} and \ref{lem:reduc2}. To avoid an uninteresting case distinction, we will also make the following assumption, which is without loss of generality, since the regret is zero while $\g_t = \vzero$.
\begin{assumption}
  \label{assum:assum1}
  We assume that $L_1 = \norm{\g_1} > 0$.
\end{assumption}

\subsection{\freegrad: An Adaptive Scale-Free Algorithm}\label{sec:commonVgrad}
In this subsection, we design a new algorithm based on a time-varying potential function, where the outputs of the algorithm are uniquely determined by the gradients of the potential function at its iterates---an approach used in the design of many existing algorithms \citep{CesaBianchiEtAl1997}.

Let $t\geq 1$, let $(\g_s)_{s\leq t}\subset\reals^d$ be a sequence of sub-gradients satisfying Assumption~\ref{assum:assum1}, and let $(h_t)$ be a sequence of hints satisfying Assumption \ref{assum:assum2}. Consider the following potential function:
	\begin{align}
	&\Phi_t \coloneqq S_t + \frac{h_1^2}{\sqrt{V_t}} \cdot   \exp\left( \frac{ \| \mathbf{G}_t \|^2}{2V_t + 2h_{t}  \|\mathbf{G}_t \| }  \right),  \quad t\geq 0, \label{eq:potential}   \\ 
	\text{where }\  \ \ \   &S_t \coloneqq \sum_{s=1}^t  \inner{\grad_s}{\what\w_s}, \quad \mathbf{G}_t \coloneqq \sum_{s=1}^t \grad_{s}, \quad V_t \coloneqq  h_1^2 + \sum_{s=1}^t \|\g_s\|^2 . \label{eq:statistics}
	\end{align}
        This potential function has appeared as a by-product in the analyses of previous algorithms such as the ones in \citep{cutkosky2018,cutkosky2019artificial}. The expression of $\Phi_t$ in \eqref{eq:potential} is interesting to us since it can be shown via the \emph{regret-reward duality} \citep{mcmahan2014} (as we do in the proof of Theorem \ref{thm:firstbound} below) that any algorithm which outputs vectors $(\what\w_t)$ such that $(\Phi_{t})$ is non-increasing for any sequence of sub-gradients $(\g_t)$, also guarantees a regret bound of the form \eqref{eq:freegrad0}. We will now design such an algorithm.

        \begin{definition}[\freegrad]
In round $t$, based on the sequence of past sub-gradients $(\g_s)_{s< t}$ and the available hint $h_{t} \ge L_{t}$, the {\freegrad} algorithm outputs the unconstrained iterate
	\begin{align}
	\what \w_{t} \coloneqq -\G_{t-1} \cdot    \frac{ (2V_{t-1} +h_{t} \|\G_{t-1}\|)\cdot h_1^2  }{2(V_{t-1}+h_{t} \|\G_{t-1} \|)^2 \ \sqrt{\eps V_{t-1}}} \cdot \exp\left(\frac{\|\G_{t-1}\|^2}{2 V_{t-1} + 2h_{t} \|\G_{t-1}\|} \right), \label{eq:predictunbounded}
	\end{align}  
        where $(\G_t)$ and $(V_t)$ are as in \eqref{eq:statistics}.
      \end{definition}

      The prediction \eqref{eq:predictunbounded} is forced by our design goal of decreasing potential $\Phi_t \le \Phi_{t-1}$. To see why, observe that at zero $\g_t=\bm{0}$ we have $\Phi_t = \Phi_{t-1}$. The weights $\what\w_{t}$ cancel the derivative $\nabla_{\g_{t}}\Phi_{t}$ at $\g_{t}=\bm{0}$, ensuring there is no direction of linear increase (which would present a violation for tiny $\g_t$).

Our main technical contribution in this subsection is to show that, in fact, with the choice of $(\what\w_{t})_{t\geq 1}$ as in \eqref{eq:predictunbounded}, the potential functions $(\Phi_t)$ are non-increasing for any sequence of sub-gradients $(\g_t)$:
	\begin{theorem}	
		\label{thm:unconstlearner}
		For $(\what\w_{t})$, and $(\Phi_t)$ as in \eqref{eq:predictunbounded}, and \eqref{eq:potential}, under Assumptions~\ref{assum:assum2} and~\ref{assum:assum1}, we have:
		\begin{align}
		\Phi_T \leq \dots \leq \Phi_0 =h_1, \quad \text{for all $T\geq 1$.}
		\end{align}
	\end{theorem}
The proof of the theorem is postponed to Appendix \ref{sec:potentialdiff}. Theorem \ref{thm:unconstlearner} and the regret-reward duality \citep{mcmahan2014} yield a regret bound for {\freegrad}. In fact, if $\Phi_T \leq \Phi_0$, then by the definition of $\Phi_T$ in \eqref{eq:potential}, we have \begin{align}\hspace{-0.4cm}
\sum_{t=1}^T \inner{\g_t}{\what\w_t}  \leq  \Phi_0 -\Psi_T(\G_T),  \ \ \text{where} \ \  \Psi_T(\G)\coloneqq  \frac{h_1^2}{\sqrt{\eps V_T}}    \exp\left( \frac{ \| \mathbf{G} \|^2}{2V_T + 2h_{T}  \|\mathbf{G} \| }  \right),\ \  \G \in \reals^d. \label{eq:dual}  \end{align}Now by  Fenchel's inequality, we have $-\Psi_T(\G_T) \leq  \inner{\w}{\G_T} +\Psi^\star_T(-\w)$, for all $\w\in \reals^d$, where  $\Psi^{\star}_T(\w)\coloneqq \sup_{\z \in \reals^d}\{ \inner{\w}{\bm{z}} -\Psi_T(\z)\}$, $\w\in \reals^d$, is the Fenchel dual of $\Psi_T$ \citep{Hiriart-Urruty}. Combining this with \eqref{eq:dual}, we obtain:
\begin{align}
\sum_{t=1}^T \inner{\g_t}{\what\w_t} \leq  \inf_{\w\in \reals^d}\left\{ \sum_{t=1}^T \inner{\g_t}{\w} + \Psi^\star_T(-\w) +\Phi_0 \right\}. \label{eq:regret}
\end{align}
Rearranging \eqref{eq:regret} for a given $\w\in \reals^d$ leads to a regret bound of $\Psi^\star_T(-\w)+\Phi_0$. Further bounding this quantity using existing results due to \cite{cutkosky2018,cutkosky2019artificial,mcmahan2014}, leads to the following regret bound (the proof is in Appendix \ref{sec:epsilon}):
\begin{theorem}
  \label{thm:firstbound}
  Under Assumptions~\ref{assum:assum2} and~\ref{assum:assum1},
for $(\what\w_{t})$ as in \eqref{eq:predictunbounded}, we have, with $\ln_+ (\cdot) \coloneqq 0\vee \ln (\cdot)$,
	\begin{align*}
	\sum_{t=1}^T \inner{\g_t}{\what \w_t-\w}  \leq  \left[ 2 \|\w\| \sqrt{V_T\ln_+ \left(\frac{ 2\|\w\| V_T }{h_1^2} \right)}   \right]  \vee \left[ 4 h_T \|\w\|  \ln \left(  \frac{4 h_T \|\w\| \sqrt{V_T}  }{ h_1^2} \right)   \right]+h_1 , 
	\end{align*}
	for all $\w \in \cW = \reals^d, T\geq 1$.
\end{theorem}
\paragraph{Range-Ratio Problem.} 
While the outputs $(\what\w_t)$ in \eqref{eq:predictunbounded} of {\freegrad} are scale-free for the sequence of hints $(h_t)$ satisfying Assumption \ref{assum:assum2}, there remains one serious issue; the fractions $V_T/h_1^2$ and $h_T/h_1$ inside the log-terms in the regret bound of Theorem \ref{thm:firstbound} could in principle be arbitrarily large if $h_1$ is small enough relative to $h_T$. Such a problematic ratio has appeared in the regret bounds of many previous algorithms which attempt to adapt to the Lipschitz constant $L$ \citep{ross2013normalized,Wintenberger2017,Kotlowski17,mhammedi19,KempkaKW19}.

When the output set $\cW$ is bounded with diameter $D>0$, this ratio can be dispensed of using a recently proposed restart trick due to \cite{mhammedi19}, which restarts the algorithm whenever $L_t/L_1> \sum_{s=1}^t \|\g_s\|/L_s$. The price to pay for this is merely an additive $O(LD)$ in the regret bound. However, this trick does not directly apply to our setting since in our case $\cW$ may be unbounded. Fortunately, we are able to extend the analysis of the restart trick to the unbounded setting where a sequence of hints $(h_t)$ satisfying Assumption \ref{assum:assum2} is available; the cost we incur in the regret bound is an additive lower-order $\wtilde{O}(\|\w\| L)$ term. Algorithm \ref{alg:wrapper} displays our restart ``wrapper'', \freerange, which uses the outputs of {\freegrad} to guarantee the following regret bound (the proof is in Appendix \ref{sec:proofs}):
\begin{theorem}
	\label{thm:freegrad1}
Let $(\what\w_t)$ be the outputs of \freerange{} (Algorithm \ref{alg:wrapper}). Then, 
\begin{gather*}
\sum_{t=1}^T \inner{\g_t}{\what \w_t-\w}  \leq 2 \|\w\| \sqrt{2V_T\ln_+ \left(\|\w\| b_T \right)}     +  h_T\cdot (16 \|\w\| \ln_+ (2 \|\w\| b_T)+2\|\w\| + 3),
\end{gather*}
for all $\w \in \reals^d, T\geq 1$, and $(\g_t)\subset \reals^d$, where $b_T\coloneqq  2\sum_{t=1}^T\del[\big]{\sum_{s=1}^{t-1} \frac{\|\g_s\|}{h_s}+2}^2\leq (T+1)^3$.
	\end{theorem} 

\begin{algorithm}[btp]
	\caption{\freerange: A Restart Wrapper for the Range-Ratio Problem (under Assumption~\ref{assum:assum1}).}
	\label{alg:wrapper}
	\begin{algorithmic}[1]
		\REQUIRE Hints $(h_t)$ satisfying Assumption \ref{assum:assum2}.
		\STATE Set $\tau=1$; \label{line:1}
		\FOR{$t=1,2,\dots$}
		\STATE Observe hint $h_{t}$;
		\IF{${h_{t}}/{h_{\tau}} >    \sum_{s=1}^{t-1} {\|\g_s\|}/{h_s} +2$}
		\STATE Set $\tau=t$;
		\ENDIF  \label{line:6}
		\STATE  Output $\what\w_{t}$ as in \eqref{eq:predictunbounded} with $(h_1,V_{t-1},\G_{t-1})$ replaced by $(h_\tau$,\ $h_\tau^2 +\sum_{s=\tau}^{{t-1}}\|\g_s\|^2$, \ $\sum_{s=\tau}^{{t-1}}\g_s$); \label{line:7}
		\ENDFOR
	\end{algorithmic}
\end{algorithm}

\noindent
We next introduce our second algorithm, in which the variance is only measured in the comparator direction; the algorithm can be viewed as a ``full-matrix'' version of \freegrad{}.
\subsection{\matfreegrad: Adapting to Directional Variance}\label{sec:multidim}

Reflecting on the previous subsection, we see that the potential function that we ideally would like to use is $S_t+ h_1 \exp \del*{\frac{1}{2} \G_t^\top \bm V_t^{-1} \G_t - \frac{1}{2} \ln \det \bm V_t}$, $t\geq 1$, where $\bm V_t =\sum_{s=1}^t \g_s \g_s^\top$. However, as we saw, this is a little too greedy even in one dimension, and we need to introduce some slack to make the potential controllable. In the previous subsection we did this by increasing the scalar denominator from $V$ to $V + \norm{\G}$, which acts as a barrier function restricting the norm of $\what\w_t$. In this section, we will instead employ a hard norm constraint. We will further need to include a fudge factor $\gamma > 1$ multiplying $\bm V$ to turn the above shape into a bona fide potential. To describe its effect, we define
  \begin{align}
    \rho(\gamma)
    \coloneqq 
    \frac{1}{
      2 \gamma
    }\del[\Big]{
      \sqrt{(\gamma +1)^2-4 e^{\frac{1}{2 \gamma } - \frac{1}{2}} \gamma ^{3/2}}
      + \gamma
      - 1
    }, \quad \text{for $\gamma \ge 1$.} \label{def.xi}
    \end{align}
The increasing function $\rho$ satisfies $\lim_{\gamma \to 1} \rho(\gamma) = 0$, $\lim_{\gamma \to \infty} \rho(\gamma) = 1$, and $\rho(2) = 0.358649$.

\bigskip\noindent
The potential function of this section is parameterized by a \emph{prod factor} $\gamma > 1$ (which we will set to some universal constant). We define
\begin{align}
  \Psi(\G, \bm V, h)
  \coloneqq
   \frac{h_1 \exp \del*{
      \inf_{\lambda \ge 0} \set*{
        \frac{1}{2} \G^\top \del*{\gamma h_1^2 \bm I + \gamma \bm V + \lambda \bm I}^{-1} \G
        + \frac{\lambda \rho(\gamma)^2}{2 h^2}
      }
    }
  }{
    \sqrt{\det\del*{\bm I + \frac{1}{h_1^2} \bm V}}
  }
  ,
  \label{eq:matrixpotential}
\end{align}
where $\G \in \reals^d$, $\bV \in \reals^{d\times d}$, and $h>0$. We introduce the following algorithm to control $\Psi$.
\begin{definition}[\matfreegrad]
In round $t$, given past sub-gradients $(\g_s)_{s< t}$ and a hint $h_{t}\geq L_T$, the {\matfreegrad} prediction is obtained from the gradient of $\Psi$ in the first argument by
\begin{equation}\label{eq:FTLR.multid}
  \what\w_{t}
  ~\df~
  - \nabla^{(1,0,0)} \Psi(\G_{t-1}, \bm V_{t-1}, h_{t}),
\end{equation}
where $\G_{t-1} = \sum_{s=1}^{t-1} \g_s$ and $\bm V_{t-1} \coloneqq \sum_{s=1}^{t-1} \g_s \g_s^\top$.
\end{definition}
We can compute $\what\w_{t}$ in $O(d^3)$ time per round by first computing an eigendecomposition of $\bm V_{t-1}$, followed by a one-dimensional binary search for the $\lambda_\star$ which achieves the $\inf$ in \eqref{eq:matrixpotential} with $(\G,\bV, h) = (\G_{t-1},\bV_{t-1},h_t)$. Then the output is given by \begin{align}\what\w_t = - \Psi(\G_{t-1}, \bm V_{t-1}, h_t) \cdot
  \del*{\gamma h_1^2 \bm I + \gamma \bm V_{t-1} + \lambda_\star \bm I}^{-1} \G_{t-1}.\end{align}
Our heavy-lifting step in the analysis is the following, which we prove in Appendix~\ref{sec:pf.multidim}:
\begin{lemma}\label{lemma:multidim.control}
  For any vector $\g_t\in \reals^d$ and $h_t>0$ satisfying $\norm{\g_t} \leq h_t$, the vector $\what\w_t$ in \eqref{eq:FTLR.multid} ensures
  \[
    \g_t^\top \what\w_t
    ~\le~
    \Psi(\G_{t-1}, \bm V_{t-1}, h_t) -
    \Psi(\G_t, \bm V_t, h_t).
  \]
\end{lemma}
From here, we obtain our main result using telescoping and regret-reward duality:
\begin{theorem}\label{thm:multidim}
  Let $\bm \Sigma^{-1}_T \df \gamma h_1^2 \bm I + \gamma \bm V_T$. For $(\what\w_t)$ as in \eqref{eq:FTLR.multid}, we have
  \[
    \sum_{t=1}^T \tuple*{\what\w_t - \w, \g_t}
    ~\le~
  h_1
  +
  \sqrt{Q_T^\w \ln_+\del*{\frac{
      \det\del*{\gamma h_1^2 \bm \Sigma_T}^{-1}
    }{
      h_1^2
    } Q_T^\w}},\quad \text{for all $\w\in \reals^d$, where}
\]
\[
  Q_T^\w
 \coloneqq 
  \max \set*{
    \w^\top \bm \Sigma^{-1}_T \w
    ,
    \frac{1}{2} \del*{\frac{h_T^2 \norm{\w}^2}{\rho(\gamma)^2} \ln \del*{
        \frac{
          \det\del*{\gamma h_1^2 \bm \Sigma_T}^{-1}
        }{
          h_1^2
        }
        \frac{h_T^2 \norm{\w}^2}{\rho(\gamma)^2}
      }
      + \w^\top \bm \Sigma^{-1}_T \w}
  }
  .
\]
\end{theorem}
Note in particular that the result is scale-free. Expanding the main case of the theorem (modest $\norm{\w}$), we find regret bounded by
\[
  \sum_{t=1}^T \tuple*{\what\w_t - \w, \g_t}
  ~\le~
  h_1
  +
  h_1
  \sqrt{\gamma \w^\top \bm Q \w \ln_+\del*{\gamma
      \w^\top \bm Q \w
      \det \bm Q
    }
  }
  \quad
  \text{where}
  \quad
  \bm Q = \bm I + \bm V_T/h_1^2
  .
\]
This bound looks almost like an ideal upgrade of that in Theorem \ref{thm:firstbound}, though technically, the bounds are not really comparable since the $\ln \det \bm Q$ can be as large as $d \ln T$, potentially canceling the advantage of having $\w^\top \bm Q \w$ instead of $\|\w\|^2 \sum_{t=1}^T \|\g_t\|^2$ inside the square-root. The matrix $\bm Q$ and hence any directional variance $\w^\top \bm Q \bm w$ is scale-invariant. The only fudge factor in the answer is the $\gamma > 1$. We currently cannot tolerate $\gamma = 1$, for then $\rho(\gamma) = 0$ so the lower-order term would explode. We note that a bound of the form given in the previous display, with the $\ln \det \bm Q$ replaced by the larger term $d \ln \op{tr} \bm Q$, was achieved by a previous (not scale-free) algorithm due to \cite{cutkosky2018}.
\begin{remark}
	\label{rem:restartremark}
	As Theorem~\ref{thm:freegrad1} did in the previous subsection, our restarts method allows us to get rid of problematic scale ratios in the regret bound of Theorem \ref{thm:multidim}; this can be achieved using \freerange{} with $(\what\w_t)$ set to be as in \eqref{eq:FTLR.multid} instead of \eqref{eq:predictunbounded}. The key idea behind the proof of Theorem~\ref{thm:freegrad1} is to show that the regrets from all but the last two epochs add up to a lower-order term in the final regret bound. This still holds when $(\what\w_t)$ are the outputs of \matfreegrad\ instead of \freegrad, since by Theorem \ref{thm:multidim}, the regret bound of \matfreegrad\ is of order at most $d$ times the regret of \freegrad\ within any given epoch.
	\end{remark}
As a final note about the algorithm, we may also develop a ``one-dimensional'' variant by replacing matrix inverse and determinant by their scalar analogues applied to $V_T = \sum_{t=1}^T \norm{\g_t}^2$. One effect of this is that the minimization in $\lambda$ can be computed in closed form. The resulting potential and corresponding algorithm and regret bound are very close to those of Section~\ref{sec:commonVgrad}.

\paragraph{Conclusion}
The algorithms designed in this section can now be used in the role of algorithm \textsc A in the reductions presented in Section~\ref{sec:helpful}. This will yield algorithms which achieve our goal; they adapt to the norm of the comparator and the Lipschitz constant and are completely scale-free, for both bounded and unbounded sets, without requiring hints. We now show that the penalties incurred by these reductions are not improvable.

\section{Lower Bounds}\label{sec:lower}
As we saw in Corollary \ref{cor:instan}, given a base algorithm \textsc{A}, which takes a sequence of hints $(h_t)$ such that $h_t\geq L_t$ for all $t\geq 1$, and which suffers regret $R^{\textsc{A}}_T(\w)$ against comparator $\w\in \cW$, there exists an algorithm \textsc{B} for the setting without hints which suffers the same regret against $\w$ up to an additive penalty $L_T \|\w\|^3+ L_T \sqrt{\max_{t\in[T]} B_t},$ where $B_t = \sum_{s=1}^t \|\g_s\|/L_t$. In this section, we show that the penalty $L_T \|\w\|^3$ is not improvable if one insists on a regret bound of order $\wtilde{O}(\sqrt{T})$. We also show that it is not possible to replace the penalty $L_T \sqrt{\max_{t\in[T]} B_t}$ by the typically smaller quantity $\sqrt{V_T}$, where $V_T=\sum_{t=1}^T \|\g_t\|^2$. Our starting point is the following lemma:
\begin{lemma}
	\label{lem:firstbound}
	For all $t \ge 1$, past sub-gradients $(\g_s)_{s<t}$ and past and current outputs $(\what\w_s)_{s \le t} \in \reals^d$,
	\begin{align}
\exists \g_t\in \reals^d, \quad 	\sum_{s=1}^{t} \inner{\g_s}{\what\w_s}  \geq  \|\what\w_t \| \cdot L_t/2,    \quad \text{where $L_t=\max_{s\leq t}\|\g_s\|$.} \label{eq:baselower}
	\end{align}
	\end{lemma}
\begin{proof}
	We want to find $\g_t$ such that $\inner{\g_t}{\what\w_t}\geq  \|\w_t\| L_{t}/2 - S_{t-1}$, where $S_{t-1}\coloneqq \sum_{s=1}^{t-1}  \inner{\g_s}{\what\w_s}$. By restricting $\g_t$ to be aligned with $\what\w_t$, the problem reduces to finding $x = \|\g_t\|$ such that 
	\begin{align}
	x \|\what\w_t\| -  \left|\|\what\w_t\| \cdot (L_{t-1}\vee x)/2 -  S_{t-1}\right|\geq 0. \label{eq:ineq}
	\end{align}
	The LHS of \eqref{eq:ineq} is a piece-wise linear function in $x$ which goes to infinity as $x\to \infty$. Therefore, there exists a large enough $x\geq0$ which satisfies \eqref{eq:ineq}.
\end{proof}
Observe that if $\|\what\w_t\|\geq D_t>0$, for $t\geq 1$, then by Lemma \ref{lem:firstbound}, there exists a sub-gradient $\g_t$ which makes the regret against $\w=\bm{0}$ at round $t$ at least $D_t L/2$. This essentially means that if the sub-gradients $(\g_t)$ are unbounded, then the outputs $(\what\w_t)$ must be in a bounded set whose diameter will depend on the desired regret bound; if one insists on a regret of order $\wtilde{O}(\sqrt{T})$, then the outputs $\what\w_t, t\geq 1,$ must be in a ball of radius at most $\wtilde{O}(\sqrt{T})$. 

\cite{cutkosky2019artificial} posed the question of whether there exists an algorithm which can guarantee a regret bound of order $ \|\w\| \sqrt{V_T \ln (\|\w\|T)} + \sqrt{V_T \ln T} + L \|\w\|^3$, with $V_T=\sum_{t=1}^T \|\g_t\|^2$, while adapting to both $L$ and $\|\w\|$ (which essentially means replacing $L \sqrt{\max_{t\in[T]} B_t}$ in Corollary~\ref{cor:instan} by $\sqrt{V_T \ln T}$). Here, we ask the question whether $\|\w\| \sqrt{V_T \ln (\|\w\|T)} + \sqrt{V_T \ln T}  + L\|\w\|^\nu$ is possible for any $\nu \geq 1$. If such an algorithm exists, then by Lemma \ref{lem:firstbound}, there exists a constant $b>0$ such that its outputs $(\what \w_t)$ satisfy $\|\what\w_t\| \leq b \sqrt{V_t \ln t}/L_t$, for all $t\geq 1$. The next lemma, when instantiated with $\alpha=2$, gives us a regret lower-bound on such algorithms (the proof is in Appendix \ref{sec:lowerproof}):
\begin{lemma}
	\label{lem:seconbound}
	For all $b,c,\beta \geq 0$, $\nu\geq 1$, and $\alpha \in ]1,2]$, there exists $(\g_t)\in\reals^d$, $T\geq 1$, and $\w\in \reals^d$, such that for any sequence $(\what\w_t)$ satisfying $\|\what\w_t\| \leq b \cdot \sqrt{V_{\alpha, t} \ln(t)/L^\alpha_t}$, for all $t\in \mathbb[T]$, where $V_{\alpha,t} \coloneqq  \sum_{s=1}^t \|\g_s\|^\alpha$, we have
	\begin{align}
	\sum_{t=1}^T \inner{\what\w_t-\w}{\g_t} \geq  c \cdot \ln (1+\|\w\| T)^{\beta} \cdot (L_T \|\w\|^\nu +  L_T^{1-\alpha/2} (\|\w\|+1) \sqrt{V_{\alpha,T} \ln T}).
	\end{align}
\end{lemma} 
By combining the results of Lemma \ref{lem:firstbound} and \ref{lem:seconbound}, we have the following regret lower bound for algorithms with can adapt to both $L$ and $\|\w\|$:
\begin{theorem}
	\label{thm:lower1}
For any $\alpha \in]1,2]$, $c>0$ and $\nu\geq 1$, there exists no algorithm that guarantees, up to multiplicative log-factors in $\|\w\|$ and $T$, a regret bound of the form $c\cdot (L_T \|\w\|^\nu +  L_T^{1-\alpha/2} (\|\w\|+1) \sqrt{V_{\alpha,T} \ln T})$, for all $T\geq 1$, $\w\in \reals^d$, and $(\g_t)\subset \reals^d$, where $V_{\alpha,T} \coloneqq  \sum_{t=1}^T \|\g_t\|^\alpha$.
\end{theorem}
\begin{proof}
	By Lemma \ref{lem:firstbound}, the only candidate algorithms are those whose outputs $(\what \w_t)$ satisfy $\|\what\w_t\| \leq b \sqrt{V_{\alpha,t} \ln (t)/L^{\alpha}_t}$, for all $t\geq 1$, for some constant $b>0$. By Lemma \ref{lem:seconbound}, no such algorithms can achieve the desired regret bound.
	\end{proof}
The regret lower bound in Theorem \ref{thm:lower1} does not apply to the case where $\alpha =1$. In fact, thanks to Corollary~\ref{cor:instan} and our main algorithm in Section~\ref{sec:mainalg} (which can play the role of Algorithm \textsc{A} in Corollary \ref{cor:instan}), we know that there exists an algorithm \textsc{B} which guarantees a regret bound of order $\wtilde{O}(L_T \|\w\|^3 + \|\w\| \sqrt{V_T\ln(\|\w\| T)} + L_T \sqrt{\max_{t\in[T]} B_t})$, where $B_t = \sum_{s=1}^t \|\g_s\|/L_t$. Next we show that if one insists on a regret bound of order $\sqrt{B_T}$, or even $\sqrt{T}$ (up to log-factors), the exponent in $\|\w\|^3$ is unimprovable (the proof of Theorem \ref{thm:lower2} is in Appendix \ref{sec:proofoflower2}).
\begin{theorem}
	\label{thm:lower2}
	For any $\nu \in[1,3[$ and $c>0$, there exists no algorithm that guarantees, up to multiplicative log-factors in $\|\w\|$ and $T$, a regret bound of the form $c\cdot( L_T \|\w\|^\nu +  L_T (\|\w\|+1) \sqrt{T\ln T})$, for all $T\geq 1$, $\w\in \reals^d$, and $(\g_t)\subset \reals^d$.
	\end{theorem}

\section{Application to Learning Linear Models with Online Algorithms}
\label{sec:linear}
In this section, we consider the setting of online learning of linear models which is a special case of OCO. At the start of each round $t\geq 1$, a learner receives a feature vector $\x_t \in \cW = \reals^d$, then issues a prediction $\what y_t\in \reals$ in the form of an inner product between $\x_t$ and a vector $\what\u_t\in \reals^d$, \emph{i.e.} $\what y_t =\what\u_t^\top \x_t$. The environment then reveals a label $y_t \in \reals$ and the learner suffers loss $\ell(y_t,\what y_t)$, where $\ell\colon \reals^2 \rightarrow \reals$ is a fixed loss function which is convex and $1$-Lipschitz in its second argument; this covers popular losses such as the logistic, hinge, absolute and Huberized squared loss. (Technically, the machinery developed so far and the reductions in Section \ref{sec:helpful} allow us to handle the non-Lipschitz case).

In the current setting, the regret is measured against the best fixed ``linear model'' $\w\in \reals^d$ as
\begin{align}
\textsc{Regret}_T(\w) \coloneqq \sum_{t=1}^T \ell(y_t, \what y_t) - \sum_{t=1}^T \ell(y_t, \w^{\top} \x_t) \leq \sum_{t=1}^T \delta_t \inner{\x_t}{\what\u_t -\w}, \label{eq:regretlin}
\end{align}
where the last inequality holds for any sub-gradients $\delta_t \in \partial^{(0,1)} \ell(y_t, \what y_t)$, $t\geq 1$, due to the convexity of $\ell$ in its second argument, which in turn makes the function $f_t(\w) \coloneqq \ell(y_t, \w^{\top} \x_t)$ convex for all $\w \in \cW = \reals^d$. Here, $\partial^{(0,1)} \ell$ denotes the sub-differential of $\ell$ with respect to its second argument. Thus, minimizing the regret in~\eqref{eq:regretlin} fits into the OCO framework described in Section \ref{sec:prelim}. In fact, we will show how our algorithms from Section \ref{sec:mainalg} can be applied in this setting to yield scale-free, and even \emph{rotation-free}, (all with respect to the feature vectors $(\x_t)$) algorithms for learning linear models. These algorithms can, without any prior knowledge on $\w$ or $(\w^{\top} \x_t)$, achieve regret bounds against any $\w\in \reals^d$ matching (up to log-factors) that of OGD with optimally tuned learning rate.

As in Section \ref{sec:mainalg}, we focus on algorithms which make predictions based on observed sub-gradients ($\g_t$); in this case, $\g_t =  \x_t \delta_t \in \x_t \cdot \partial^{(0,1)} \ell(y_t, \what y_t)= \partial f_t(\what \u_t)$, $t\geq 1$, where $f_t(\w)=\ell(y_t,\w^{\top} \x_t)$. Since the loss $\ell$ is $1$-Lipschitz, we have $|\delta_t|\leq 1$, for all $\delta_t \in \partial^{(0,1)}\ell(y_t, \what y_t)$ and $t\geq 1$, and so $\|\g_t\|\leq \|\x_t\|$. Since $\x_t$ is revealed at the beginning of round $t\geq 1$, the hint 
\begin{align}h_t =  \max_{s \leq t }\|\x_s\| \geq L_T = \max_{s\leq t} \|\g_s\| \label{eq:lin2} \end{align} 
is available ahead of outputting $\what \u_t$, and so our algorithms from Section \ref{sec:mainalg} are well suited for this setting.

\paragraph{Improvement over Current Algorithms.} We improve on current state-of-the-art algorithms in two ways; First, we provide a (coordinate-wise) scale-invariant algorithm which guarantees regret bound, against any $\w\in \reals^d$, of order 
\begin{align}
\sum_{i=1}^d |w_i| \sqrt{V_{T,i} \ln (|w_i| \sqrt{V_{T,i}}T)} + |w_i| \ln_+(|w_i|\sqrt{V_{T,i}}T),  \label{eq:olbound}
\end{align}
where $V_{T,i}\coloneqq |x_{1,i}|^2 + \sum_{t=1}^T \delta^2_t|x_{t,i}|^2, i\in[d]$, which improves the regret bound of the current state-of-the-art scale-invariant algorithm $\textsc{ScLnOL}_1$ \citep{KempkaKW19} by a $\sqrt{\ln (\|\w\| T)}$ factor. Second, we provide an algorithm that is both scale and rotation invariant with respect to the input feature vectors $(\x_t)$ with a state-of-the-art regret bound; by scale and rotation invariance we mean that, if the sequence of feature vectors $(\x_t)$ is multiplied by $c \bm{O}$, where $c >0$ and $\bm{O}$ is any special orthogonal matrix in $\reals^{d\times d}$, the outputs ($\what y_t$) of the algorithm remain unchanged. Arguably the closest algorithm to ours in the latter case is that of \cite{Kotlowski17} whose regret bound is essentially of order $\wtilde{O}(\sqrt{\w^{\top} \bm{S}_T \w})$ for any comparator $\w\in \reals^d$, where $\bm{S}_T = \sum_{t=1}^T \x_t \x_t^\top$. However, in our case, instead of the matrix $\bm{S}_T$, we have $\bV_T \coloneqq  \|\x_1\|^2  \bm I   + \sum_{t=1}^T \x_t \x_t^\top \delta_t^2 $, where $\delta_t \in \partial^{(0,1)} \ell(y_t, \what y_t), t\geq 1$, which can yield a much smaller bound for small $(\delta_t)$ (this typically happens when the algorithm starts to ``converge'').

\paragraph{A Scale-Invariant Algorithm.} To design our first scale-invariant algorithm, we will use the outputs ($\what\w_t$) of \freegrad{} in \eqref{eq:predictunbounded} with $(h_t)$ as in \eqref{eq:lin2}, and a slight modification of \freerange{} (see Algorithm \ref{alg:newfreerange}). This modification consists of first scaling the outputs $(\what \w_t)$ of \freegrad\ by the initial hint of the current epoch to make the predictions $(\what y_t)$ scale-invariant. By Theorem \ref{thm:scaled} below, the regret bound corresponding to such scaled outputs will have a lower-order term which, unlike in the regret bound of Theorem \ref{thm:firstbound}, does not depend on the initial hint. This breaks our current analysis of \freerange{} in the proof of Theorem \ref{thm:freegrad1} which we used to overcome the range-ratio problem. To solve this issue, we further scale the output $\what\w_t$ at round $t\geq 1$ by the sum $\sum_{s=1}^\tau \|\x_s\|/h_s$, where $\tau$ denotes the first index of the current epoch (see Algorithm \ref{alg:newfreerange}). Due to this change, the proof of the next theorem differs slightly from that of Theorem \ref{thm:freegrad1}.

 First, we study the regret bound of Algorithm \ref{alg:newfreerange} in the case where $\cW=\reals$.
\begin{theorem}
	\label{thm:freegradol}
	Let $d=1$ and $(h_t)$ be as in \eqref{eq:lin2}. If $(\what u_t)$ are the outputs of Algorithm \ref{alg:newfreerange}, then for all $w \in \reals; T\geq 1$; $(x_t,y_t)\subset \reals^2$, s.t. $h_1=|x_1|> 0$; and $\delta_t \in \partial^{(0,1)} \ell\left(y_t,x_t \what u_t \right)$, $t\in[T]$,
	\begin{align}
\sum_{t=1}^T \delta_t x_t \cdot \left(\what u_t- w\right) & \leq  2 |w| \sqrt{V_{T}\ln_+ (2 |w|^2 V_{T} c_T )}  \\ & \quad  +  h_{T} |w| ( 14 \ln _+(  2 |w| \sqrt{2V_{T} c_T }  ) + 1 )+  {2+ \ln B_T}, \label{eq:regbound}
	\end{align}
where $V_T \coloneqq |x_1|^2 +\sum_{t=1}^T \delta_t^2 x_t^2$, $c_T\coloneqq  2 B_T^2  \sum_{t=1}^T\left(\sum_{s=1}^t \frac{|x_s |}{h_s}\right)^2 \leq T^5$, and $B_T =\sum_{s=1}^{T} \frac{|x_s |}{h_s}\leq T$.
\end{theorem}
The proof of Theorem \ref{thm:freegradol} is in Appendix \ref{sec:sec5proofs}. If $(\what u_t)$ are the outputs of Algorithm \ref{alg:newfreerange} in the one-dimensional case, then by Theorem \ref{thm:freegradol} and \eqref{eq:regretlin}, the algorithm which, at each round $t\geq 1$, predicts $\what y_t=x_t \what u_t$ has regret bounded from above by the RHS of \eqref{eq:regbound}. Note also that the outputs $(\what y_t)$ are scale-invariant.

Now consider an algorithm $\textsc{A}$ which at round $t\geq 1$ predicts $\what y_t =\sum_{i=1}^d x_{t,i} \what u_{t,i}$, where $(\what u_{t,i}), i\in[d]$, are the outputs of Algorithm \ref{alg:newfreerange} when applied to coordinate $i$; in this case, we will have a sequence of hints $(h_{t,i})$ for each coordinate $i$ satisfying $h_{t,i} = \max_{s\leq t}|x_{t,i}|$, for all $t\geq 1$. Algorithm $\textsc{A}$ is coordinate-wise scale-invariant, and due to \eqref{eq:regretlin} and Theorem \ref{thm:freegradol}, it guarantees a regret bound of the form \eqref{eq:olbound}. We note, however, that a factor $d$ will appear multiplying the lower-order term $(2+\ln B_T)$ in \eqref{eq:regbound} (since the regret bounds for the different coordinates are added together). To avoid this, at the cost of a factor $d$ appearing inside the logarithms in \eqref{eq:olbound}, it suffices to divide the outputs of algorithm \textsc{A} by $d$. To see why this works, see Theorem \ref{thm:scaled} in the appendix.

\begin{algorithm}[btp]
	\caption{Modified \freerange\ for the setting of online learning of linear models.}
	\label{alg:newfreerange}
	\begin{algorithmic}[1]
		\REQUIRE The hints $(h_t)$ as in \eqref{eq:lin2}.
		\STATE Set $\tau=1$; 
		\FOR{$t=1,2,\dots$}
		\STATE Observe hint $h_{t}$;
		\IF{${h_{t}}/{h_{\tau}} >    \sum_{s=1}^{t-1} {\|\x_s\|}/{h_s} +1$}
		\STATE Set $\tau=t$;
		\ENDIF  
		\STATE  Output $\what\u_t= \what\w_{t} \cdot \left(h_\tau  \cdot {\sum_{s=1}^{\tau } \frac{\|\x_s\|}{h_s}}\right)^{-1}$, where $\what \w_t$ is as in \eqref{eq:predictunbounded} with $(h_1, V_{t-1}, \G_{t-1})$ replaced by $(h_{\tau}, h_\tau^2+ \sum_{s=\tau}^{t-1} \|\g_s\|^2,\ \sum_{s=\tau}^{t-1}\g_s)$; \label{eq:lastline}
		\ENDFOR
	\end{algorithmic}
\end{algorithm}
\paragraph{A Rotation-Invariant Algorithm.} To obtain a rotation and scale-invariant online algorithm for learning linear models we will make use of the outputs of \matfreegrad{} instead of \freegrad{}. Let $(\what y_t)$ be the sequence of predictions defined by
\begin{align}\hspace{-0.2cm}\what y_t=\x_t^\top \what\w_t/h_{1}, \  t\geq 1, \label{eq:pred2}
\end{align}
with $(h_t)$ as in \eqref{eq:lin2} and where $\what\w_t$ are the predictions of a variant of \matfreegrad{}, where the leading $h_1$
in the potential \eqref{eq:matrixpotential} is replaced by $1$ (we analyze this variant in Appendix~\ref{sec:multidim.control.real}).
\begin{theorem}\label{thm:multidimlinear}
	Let $\gamma>0$ and $(h_t)$ be as in \eqref{eq:lin2}. If $(\what y_t)$ are as in \eqref{eq:pred2}, then
	\[
\forall \w \in \reals^d, \forall T\geq 1, \forall (\g_t)\subset \reals^d, \ \ \textsc{Regret}_T(\w)
	~\le~
	1
	+
	\sqrt{Q_T^\w \ln_+\del*{{
				\det\del*{ \gamma h_1^2 \bm \Sigma_T}^{-1}
			} Q_T^\w}},\ \ \text{where}
	\]
	\[
	Q_T^\w
	\coloneqq 
	\max \set*{
		\w^{\top} \bm \Sigma^{-1}_T \w
		,
		\frac{1}{2} \del*{\frac{h_T \norm{\w}^2}{\rho(\gamma)^{2}} \ln \del*{
				\frac{h_T \norm{\w}^2}{\rho(\gamma)^{2}}
				\det \left(\gamma  h^2_1 \bm \Sigma_T\right)^{-1}
			}
			+ \w^{\top} \bm \Sigma^{-1}_T \w}
	}
	,
	\]
and $\bm \Sigma^{-1}_T \df  \gamma h_1^2 \bm I + \gamma \sum_{t=1}^T \g_t \g_t^\top$.
\end{theorem}
\begin{proof}
It suffices to use \eqref{eq:regretlin} and instantiate the regret bound in Theorem~\ref{thm:rephrased} with $(\epsilon,\sigma^{-2})= (1,\gamma h^2_1)$.
\end{proof}
The range-ratio problem manifests itself again in Theorem \ref{thm:multidimlinear} through the term $	\det(\gamma  h^2_1 \bm \Sigma_T)^{-1}$. This can be solved using the outputs of Algorithm \ref{alg:newfreerange}, where in Line \ref{eq:lastline}, $\what \w_t$ is taken to be as in \eqref{eq:pred2} (see Remark \ref{rem:restartremark}).
	\newpage

		\acks{
		This work was supported by the Australian Research Council and
		Data61.}
	
	\DeclareRobustCommand{\VAN}[3]{#3} 
	\bibliography{biblio}
	
	\DeclareRobustCommand{\VAN}[3]{#2} 
	\clearpage

	\begin{appendix}
	\section{Proof of Theorem \ref{thm:unconstlearner}}
\label{sec:potentialdiff}
The proof of Theorem \ref{thm:unconstlearner} relies on the following key lemma:
\begin{lemma}
	\label{lem:secondkeylem}
	For $G,g\in \reals$ and, $V>0$, define
	\begin{align}
	\Theta(G,V,g)\coloneqq \frac{\sqrt{V}}{\sqrt{V+g^2}}\cdot \exp\left({\frac{(  G+  g)^2}{2V+2g^2 + 2|  G+  g|} } -\frac{G^2}{  2V +  2|G|} \right)   -  \frac{   g   G   (  |G|  +2   V)}{2(  |G| +  V)^2} -1. \label{eq:keyeq}
	\end{align}
	It holds that $\Theta(G,V,g)\leq 0$, for all $G \in \reals$, $V> 0$, and $g \in[-1,1]$.
\end{lemma}
\begin{proof}
	For notational simplicity we assume $G \ge 0$. Let us look at
	\begin{align}
	\label{eq:thegamma}
	\Gamma(G,V,g)
	~\df~
	\frac{1}{2}\frac{(g+G)^2}{V+g^2 + |  G+  g|}
	-\frac{1}{2}\frac{G^2}{G +V}
	- \ln \del*{1+\frac{g G (G +2 V)}{2 (G +V)^2}}
	-\frac{1}{2} \ln \del*{1+\frac{g^2}{V}}.
	\end{align}
	Since $\ln$ is increasing, we have that $\Theta \leq 0$, if and only if, $\Gamma \leq 0$, and so we want to show $\Gamma \le 0$ for all $V > 0, G \ge 0$, and $g \in [-1,1]$. Our approach will be to show that $\Gamma$ is increasing in $V$. The result then follows from $\lim_{V \to \infty} \Gamma ~=~ 0$. It remains to study the derivative
	\begin{align}
	\frac{\partial  \Gamma}{\partial V}
	&~=~
	- \frac{1}{2} \frac{(g+G)^2}{\left(\left| g+G\right| +g^2+V\right)^2}
	+ \frac{2 g G V}{(G +V) (2 (G +V)^2+g G (G +2 V))} \\ 
	& \quad + \frac{1}{2} \frac{G^2}{(G+V)^2}	+ \frac{1}{2} \frac{g^2}{g^2 V+V^2}.
	\end{align}
	Factoring this as a ratio of polynomials, we obtain:
	\begin{align}
	\frac{\partial  \Gamma}{\partial V}
	~=~
	\frac{
		\alpha_0 + \alpha_1 V + \alpha_2 V^2 + \alpha_3 V^3 + \alpha_4 V^4 + \alpha_5 V^5
	}{
		2 V \left(g^2+V\right) (G+V)^2 \left((g+2) G^2+2 (g+2) G V+2 V^2\right) \left(\left| g+G\right| +g^2+V\right)^2
	},
	\end{align}
	where $\alpha_i, i\in[5]$, are polynomials in $g$ and $G$ whose explicit (yet gruesome) expressions are:
	\begin{align*}
	\alpha_0 &~=~ g^2 (g+2) G^4 \left(\left| g+G\right| +g^2\right)^2
	\\
	\alpha_1 &~=~ g^2 (g+2) G^3 \left(2 \left(g^2 (G+4)+G\right) \left| g+G\right| +g^4 (G+4)+2 g^2 (G+2)+8 g G+4 G^2\right)
	\\
	\alpha_2 &~=~ g^2 G^2 \left(
          \begin{aligned}
            &2 \left(g^3 (2 G+9)+4 g^2 (G+3)+2 g G (G+2)+4 G (G+2)\right) \left| g+G\right| \\
            &+3 g^2 (3 g+4) \left(g^2+1\right)-2 (g+2) G^3\\
            &+(g (g (3 g+2)+2)+14) G^2+2 g (g (g (g (g+2)+3)+15)+12) G
          \end{aligned}
        \right)
	\\
	\alpha_3 &~=~ G \left(
          \begin{aligned}
            &2 \left(6 g^5+2 g^4 (G+4)+g^3 G (4 G+13)+4 g^2 G (2 G+3)+g G^3+2 G^3\right) \left| g+G\right| \\
            &+(g (g+2) (2 g+13)+20) g^3 G+2 \left(g \left(3 g^2+g-6\right)+10\right) g^2 G^2\\
            &+2 (3 g+4) \left(g^2+1\right) g^4-2 (g+2) G^4-2 (g (g+4)+2) g G^3
          \end{aligned}
        \right)
	\\
	\alpha_4 &~=~ 2 \left(
          \begin{aligned}
            &2 \left(g^4+(5 g+4) g^2 G+(g+2) G^3+2 (g+1) g G^2\right) \left| g+G\right| +g^6 +g^4\\
            & +(g (7 g+4)+4) g^3 G+((g-2) g+6) g^2 G^2-2 (g+2) G^4+(g (2 g-3)-8) g G^3
          \end{aligned}
        \right)
	\\
	\alpha_5 &~=~ 2 \left(2 \left| g+G\right| ^3+g^4+4 g^3 G-6 g G^2-2 G^3\right)
      \end{align*}
      Under our assumptions $V>0$, $G>0$ and $g\in [-1,1]$, the denominator of $\partial \Gamma/\partial V$ above is positive. Furthermore, its numerator, regarded as a polynomial in $V$, has exclusively positive coefficients $\alpha_i \ge 0$, as can be verified using computer algebra software (we used Mathematica's \texttt{FullSimplify}---see Figure \ref{fig:mat}). This implies that $\partial \Gamma/\partial V\geq 0$, for all $G\in \reals$, $V> 0$, and $g\in [-1,1]$, and so $\Gamma \leq \lim_{V\to \infty}  \Gamma =0$.
\end{proof}

	\begin{figure}
	\includegraphics[trim=2cm 9.5cm 2cm 2cm, clip, width=\textwidth]{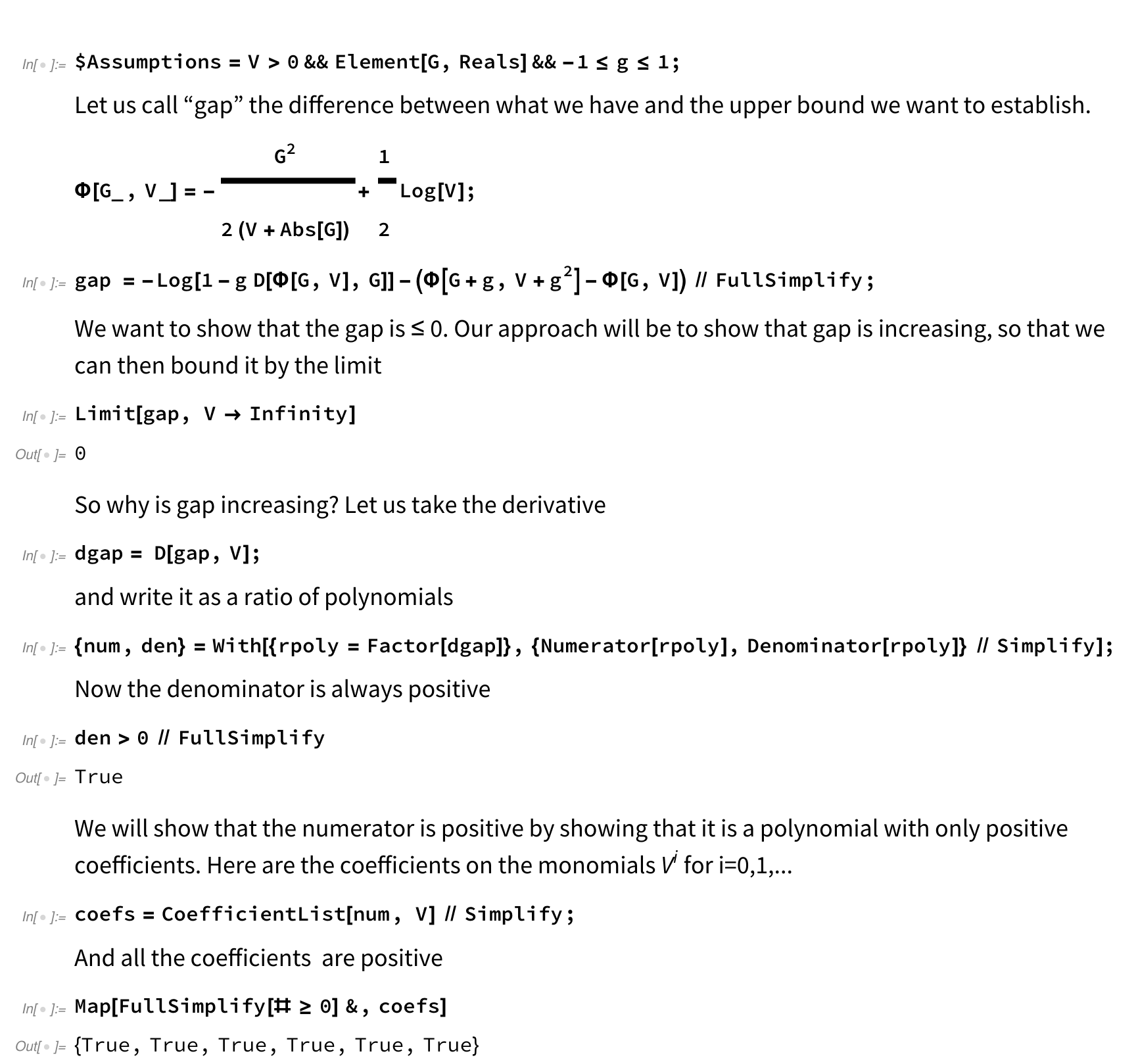}
	\caption{Mathematica notebook in support of Theorem~\ref{thm:unconstlearner}.}
	\label{fig:mat}
\end{figure}

We are ourselves a bit disgruntled about the opacity of the above proof. On the one hand, it is just a tedious verification of an analytic statement about a function of three scalar variables, and one might expect that tighter statements require more sophisticated techniques \cite[c.f.][Appendix~F]{Kotlowski17}. It is quite plausible that positivity may be established in a somewhat more streamlined fashion using Sum-of-Squares techniques. Yet on the other hand, we were hoping to gain, from the proof, a deeper insight into the design of potential functions. Unfortunately this did not materialize. In particular, we still do not know how to address the multi-dimensional case of our Section~\ref{sec:multidim} with a similar potential. Controlling the intuitive upgrade of \eqref{eq:potential} where the exponential is replaced by
\[
  \exp \del*{
    \sup_{\u} \frac{(\u^\top \G)^2}{2 (\u^\top \bm V \u + L \norm{\u} \cdot \abs{\u^\top \G})}
    - \frac{1}{2} \ln \det \bm V
  }
\]
is impossible, as witnessed by numerical counterexamples returned by random search, already in dimension $2$.

We need one more result before we prove Theorem \ref{thm:unconstlearner}:
\begin{lemma}
	\label{lem:firstterm}
Let $G, s \in \reals$ and $V, h\geq 0$. Then, the function
\begin{align}
g \mapsto \frac{1}{\sqrt{V + g^2}} \exp \left( \frac{g^2+ 2 s + G^2}{2 V + 2 g^2 + 2 h \sqrt{g^2 + 2 s + G^2}}  \right),
\end{align}	
is non-increasing on $\{ g\geq 0 \mid  g^2 +2 s + G^2\geq 0 \}$.
\end{lemma}
\begin{proof}
	It suffices to show that the function
	\begin{align}
	  \Xi(g) \coloneqq  \frac{g^2+ 2 s + G^2}{2 V + 2 g^2 + 2 h \sqrt{g^2 + 2 s + G^2}} - \frac{1}{2}\ln \left(V + g^2 \right),
	\end{align}	
	is non-increasing on $\{ g\geq 0 \mid  g^2 +2 s + G^2\geq 0 \}$. Evaluating the derivative of $\Xi$, we find that 
	\begin{align}
	\frac{\mathrm d  \Xi}{\mathrm d  g}(g) = -\frac{g N \cdot  \left(2 h^2 N+2 V
		N+2 g^2 N+3 g^2 h+3 h V\right)}{2
		\left(g^2+V\right) \left(h N+g^2+V\right)^2}, \label{eq:deriv}
	\end{align}
	where $N\coloneqq \sqrt{g^2+G^2+2s}$. The derivative in \eqref{eq:deriv} is non-positive for all $V, h\geq 0$ and $g \geq 0$ such that $g^2 + 2 s + G^2 \geq 0$.
	\end{proof}
\bigskip
\begin{proof}{\textbf{of Theorem \ref{thm:unconstlearner}.}}
	We will proceed by induction. By the fact that $\|\G_0\| =0$ and the definition of the potential in \eqref{eq:potential}, we have $\Phi_0=h_1$. Now let $t\geq 0$, $h_1>0$, and $(S_t,V_t,h_t,\G_{t})\in  \reals \times \reals_{\geq 0} \times   \reals_{> 0}\times  \reals^d$. We will show that $\Phi_{t+1} -\Phi_t\leq 0$. First, note that for any $h_{t+1}\geq h_t$, we have
	\begin{align}
	S_t + \frac{h_1^2}{\sqrt{\eps V_t}} \cdot   \exp\left( \frac{ \| \mathbf{G}_t \|^2}{2V_t + 2h_{t+1}  \|\mathbf{G}_t \| }  \right) \leq  \Phi_t = S_t + \frac{h_1^2}{\sqrt{\eps V_t}} \cdot   \exp\left( \frac{ \| \mathbf{G}_t \|^2}{2V_t + 2h_{t}  \|\mathbf{G}_t \| }  \right). \label{eq:potineq}
	\end{align}
	Thus, for any $\g_{t+1}\in \reals^d$ such that $\|\g_{t+1}\|\leq h_{t+1}$, and $(S,V,h,\G,\g)\coloneqq (S_t,V_t,h_{t+1},\G_t,\g_{t+1})$,
	\begin{align}
	\Phi_{t+1} - \Phi_t &\leq \frac{h_1^2}{\sqrt{\eps V+\eps \|\g\|^2}} \cdot \exp \left(\frac{\|\grad+\G\|^2}{ 2 V+2\|\grad\|^2+2 h  \| \grad+\G\| }\right)  \\   &\quad  -\left(1+ \frac{  \inner{\grad}{\G}  \cdot (2 h \| \G\| +2  V)}{2(h \| \G\| + V)^2} \right) \cdot\frac{h_1^2}{\sqrt{\eps V}}  \cdot  \exp\left({\frac{\|\G\|^2}{2V+ 2h  \| \G\| }} \right). \label{eq:diff000}
	\end{align}
	Let $\g_\star$ be the vector $\g \in \mathcal{B}_h$ which maximizes the RHS of \eqref{eq:diff000}, where $\mathcal{B}_h$ is the ball in $\mathbb{R}^d$ of radius $h$. Suppose that $\G\neq \bm 0$, and let $\mathcal{H} \coloneqq \{\g \in \reals^d \mid \inner{\g}{\G} = \inner{\g_\star}{\G}\}$. Note that within the hyperplane $\mathcal{H}$, only the first term on the RHS of \eqref{eq:diff000} varies. Since $\g_\star$ is the maximizer of the RHS of \eqref{eq:diff000} within $\mathcal{B}_h$, instantiating Lemma \ref{lem:firstterm} with $s\coloneqq \inner{\g_\star}{\G}$ and $G\coloneqq \|\G\|$, implies that $\g_\star \in \argmin \{ \|\g\| \mid \g \in \mathcal{H}\}$. Adding this to the fact that $\mathcal{H}$ is a hyperplane orthogonal to $\G$ implies that $\g_\star$ and $\G$ must be aligned, \emph{i.e.} there exists a $c_\star\in \reals$ such that $\g_\star =c_\star \G/\|\G\|$.  
 Therefore, we have $\|\g_\star+\G\|=|g_\star+G|$, where  $$g_\star\coloneqq \left\{ \begin{matrix} c_\star, &  \text{if } G>0;\\   \|\g_\star\|,  & \text{otherwise}. \end{matrix}\right.$$ 
        Further, note that $|g_\star|\leq h$. Thus, the RHS of \eqref{eq:diff000} is bounded from above by
	\begin{align}
	\Delta & \coloneqq	\frac{h_1^2}{\sqrt{\eps V+\eps g_\star^2}} \cdot \exp \left(\frac{(g_\star+G)^2}{ 2 V+2 g_\star^2+2 h |g_\star+G| }\right)  \\ \quad & -\left(1+ \frac{  g_\star G \cdot (2  V+2 h |G|)}{2(V+h|G|)^2} \right) \cdot\frac{h_1^2}{\sqrt{\eps V}}  \cdot  \exp\left({\frac{G^2}{2V+ 2h  |G| }} \right).\label{eq:delta}
	\end{align}
Note that $\Delta$ in \eqref{eq:delta} can be written in terms of the function $\Theta$ in Lemma \ref{lem:secondkeylem} as:
	\begin{align}
	\Delta = \frac{h_1^2}{\sqrt{\eps V}} \cdot \exp\left(\frac{G^2}{2V+2hG}\right) \cdot \Theta\left(\frac{G}{h},\frac{ V}{h^2},\frac{g_\star}{h}\right).\label{eq:rewrite}
	\end{align}
	Since $(G/h,V/h^2,g_\star/h)\in \reals \times \reals_{>0}\times [-1,1]$, Lemma \ref{lem:secondkeylem} implies that $\Theta(G/h,V/h^2,g_\star/h)\leq 0$, and so due to \eqref{eq:delta}, we also have $\Delta\leq 0$. Since $\Delta$ is an upper-bound on the RHS of \eqref{eq:diff000}, it follows that $\Phi_{t+1}-\Phi_t\leq 0$ as desired.  
\end{proof}

	\section{Proofs of Section \ref{sec:commonVgrad}}
	\label{sec:proofs}
	\subsection{Proof of Theorem \ref{thm:firstbound}}
	\label{sec:epsilon}
The proof of Theorem \ref{thm:firstbound} follows from the next theorem by setting $\epsilon=1$. Theorem \ref{thm:scaled} essentially gives the regret bound of \freegrad\ if its outputs $(\what\w_t)$ are scaled by a constant $\epsilon>0$. This will be useful to us later.
	\begin{theorem}
		\label{thm:scaled}
	Let $\epsilon>0$, and $\what \u_t \coloneqq \what \w_t/\epsilon$, for $(\what\w_{t})$ as in \eqref{eq:predictunbounded}.	Then, under Assumptions~\ref{assum:assum2} and~\ref{assum:assum1}:
		\begin{align*}
		\sum_{t=1}^T \inner{\g_t}{\what \u_t-\w}  \leq  \left[ 2 \|\w\| \sqrt{V_T\ln_+ \left(\frac{ 2 \epsilon \|\w\| V_T }{h_1^2} \right)}   \right]  \vee \left[ 4 h_T \|\w\|  \ln \left(  \frac{4 h_T\epsilon \|\w\| \sqrt{V_T}  }{ h_1^2} \right)   \right]+\frac{h_1}{\epsilon} , 
		\end{align*}
		for all $\w \in \cW = \reals^d, T\geq 1$.
	\end{theorem}
\begin{proof}
Since the assumptions of Theorem \ref{thm:unconstlearner} are satisfied, we have 
	\begin{align}
\Phi_T = \sum_{t=1}^T \what \w_t^\top \g_t  +\frac{h_1^2}{\sqrt{\eps V_T}} \cdot   \exp\left( \frac{  \|\G_T\| ^2}{2V_T + 2h_{T} \|\G_T\| }  \right) \leq  \Phi_0=h_1, \label{eq:potentialineq0}
\end{align}
	Dividing both sides of \eqref{eq:potentialineq0} by $\epsilon>0$ and rearranging yields  
\begin{align}
\sum_{t=1}^T \what \u_t^\top \g_t  \leq  \frac{h_1}{\epsilon}-  \Theta_T(\G_T), \quad \text{where} \ \  \Theta_T(\G) \coloneqq  \frac{h_1^2}{\epsilon\sqrt{\eps V_T}} \cdot   \exp\left( \frac{  \|\G\|^2}{2V_T + 2h_{T}  \|\G\| }  \right), \G\in \reals^d, \label{eq:potentialineqq0}
\end{align}
By duality, we further have that 
\begin{align}
\sum_{t=1}^T \what \u_t^\top \g_t  \leq  \frac{h_1}{\epsilon}+ \w^\top \G_T +  \Theta^{\star}_T(-\w), \quad \text{for all } \w\in \reals^d. \label{eq:start0}
\end{align}
Since $\Theta_T(\G) = \Psi_T(\G)/\epsilon$, for all $\G\in \reals^d$, where $\Psi_T$ is the function defined in \eqref{eq:dual}, we have by the properties of the Fenchel dual \cite[Prop. 1.3.1]{Hiriart-Urruty} that 
\begin{align}
\Theta_T^\star(\w) =	\Psi^\star_T( \epsilon \w)/\epsilon, \quad \text{for all } \w \in \reals^d. \label{eq:dualed}
\end{align}
We now bound $\Psi^\star_T(-\w)$ from above, for $\w\in \reals^d$. For this, note that $\Psi_T(\G) = \psi_T(\|\G\|/h_T)$, for $\G\in \reals^d$, where 
\begin{align}\psi_T(x) \coloneqq \frac{h_1^2}{\sqrt{V_T}} \cdot \exp\left(\frac{x^2}{2V_T/h_T^2 + 2|x|}\right).\label{eq:psi} \end{align} Thus, according to \cite[Lemma 3]{mcmahan2014} and the properties of duality \cite[Prop. E.1.3.1]{Hiriart-Urruty}, we have \begin{align} \label{eq:relation}\Psi^\star_T(-\w)=   \psi^\star_T(h_T \|\w\|).\end{align} On the other hand, \cite[Lemma  18, 19]{cutkosky2018} and \cite[Lemma  18]{orabona2016coin} provides the following upper-bound on $\psi^\star_T(u)$, $u\in \reals$, using the Lambert function $W$ (where $W(x)$ is defined as the principal solution to $W(x) e^{W(x)} = x$): \begin{gather}   \psi^\star_T(u)\leq  \Lambda_T(u) \vee \left( 4 u \cdot \ln \left(  \frac{4 u \sqrt{V_T}  }{ h_1^2} \right)   \right),\label{eq:dualupper} \\
\text{where} \quad  \Lambda_T(y) \coloneqq y \sqrt{ \frac{2 V_T}{h_T}}  \cdot \left( \left(W\left(c_T^2 y^2 \right)\right)^{1/2} - \left(W\left(c_T^2 y^2 \right)\right)^{-1/2} \right), y\in \reals,
\end{gather}
and $c_T \coloneqq {\sqrt{2} V_T }/{(h_T h_1^2)}$. Using the fact that the Lambert function satisfies $(W(x))^{1/2} -(W(x))^{-1/2}\leq \sqrt{\ln_+ x}$, for all $x\geq 0$ (see Lemma \ref{eq:Xfn}), together with \eqref{eq:dualupper} and \eqref{eq:relation} implies that 
\begin{align}
\Psi_T^\star(-\w)  \leq   \left[ 2 \|\w\| \sqrt{V_T\ln_+ \left(\frac{ 2\|\w\| V_T }{h_1^2} \right)}   \right]  \vee \left[ 4 h_T \|\w\|  \ln \left(  \frac{4 h_T \|\w\| \sqrt{V_T}  }{ h_1^2} \right)   \right],\label{eq:upperbound}
\end{align}
for all $\w\in \reals^d$. Combining this with \eqref{eq:start0} and \eqref{eq:dualed} leads to the desired regret bound.
	\end{proof}

\begin{proof}{\textbf{of Theorem \ref{thm:firstbound}.}}
Invoke Theorem \ref{thm:scaled} with $\epsilon=1$.
\end{proof}

\subsection{Proof of Theorem \ref{thm:freegrad1}}
\begin{proof}
Fix $\w\in \reals^d$ and let $k\geq 1$ be the total number of epochs. We denote by $\tau_i\geq 1$ the start index of epoch $i\in[k]$. Further, for $\tau,\tau' \in \mathbb{N}$, we define $\tilde\tau \coloneqq \tau -1$ and $V_{\tau:\tau'}\coloneqq h_\tau^2+ \sum_{s=\tau}^{\tilde\tau'} \|\g_s\|^2$ (note how the upper index is exclusive). 
	Recall that at epoch $i\in[k]$, the restart condition in Algorithm \ref{alg:wrapper} is triggered at $t=\tau_{i+1}> \tau_{i}$ only if 
	\begin{align}
	\frac{h_t}{h_{\tau_i}}  > \sum_{s=1}^{t-1} \frac{\|\g_s\|}{h_s} +2\geq  \sum_{s=1}^t \frac{\|\g_s\|}{h_s}, \label{eq:cond}
	\end{align}
where the last inequality follows by Assumption \ref{assum:assum2}.	We note that \eqref{eq:cond} also implies that
\begin{align}
h_{{\tau}_{i+1}}  > 2 h_{\tau_{i}}, \quad \text{for all } i\in[k]. \label{eq:doubling}
\end{align}
On the other hand, within epoch $i\in[k]$, $\frac{h_t}{h_{\tau_i}} \leq \sum_{s=1}^{t-1} \frac{\|\g_s\|}{h_s} +2$, for all  $\tau_{i}  \leq  t\leq  \tilde{\tau}_{i+1}$, and thus
\begin{align}
 \frac{h_{\tilde\tau_{i+1}}}{h_{\tau_{i}}} \leq \frac{\sqrt{V_{\tau_{i}:\tau_{i+1}}}}{h_{\tau_i}} \leq  \sqrt{\sum_{t=\tau_{i}}^{\tilde{\tau}_{i+1}} \left(\sum_{s=1}^{t-1} \frac{\|\g_s\|}{h_s}+2\right)^2}  \leq \sqrt{\frac{b_T}{2}}, 	 \label{eq:cons1}
	\end{align}
	where $b_T \coloneqq 2 \sum_{t=1}^T (\sum_{s=1}^{t-1} {\|\g_s\|}/{h_s} +2)^2$. Therefore, by the regret bound of Theorem \ref{thm:firstbound} and \eqref{eq:cons1}:
	\begin{align}
\hspace{-0.5cm}	\sum_{s=\tau_{i}}^{\tilde\tau_{i+1}} \inner{\g_s}{\what\w_s-\w} &\leq 2 \|\w\| \sqrt{V_{\tau_i:\tau_{i+1}} \ln_+\left( \|\w\|b_T \right)} +   (4 \|\w\| \ln (2 \|\w\| b_T) + 1)h_{{\tau_{i+1}}},\ \  i\in[k]. \label{eq:epochs}
	\end{align}
	Summing this inequality over $i=1,\dots, k-2$, we get:
		\begin{align}
\hspace{-0.4cm}	\sum_{s=1}^{\tilde\tau_{k-1}} \inner{\g_s}{\what\w_s-\w} \leq 2 \|\w\|\sum_{i=1}^{k-2} \sqrt{ V_{\tau_i:\tau_{i+1}} \ln_+\left( \|\w\|b_T \right)} +  \sum_{i=1}^{k-2}   ( 4 \|\w\| \ln (2 \|\w\| b_T) + 1) h_{\tau_{i+1}} . \label{eq:epochs2}
	\end{align}
	Now using \eqref{eq:cond} at $t=\tau_{i+2}$, we have for all $i\in [k]$,  \begin{align}
	\hspace{-0.5cm}
V_{\tau_i:\tau_{i+1}} \leq 	  h^2_{\tau_{i}}+ \sum_{s=1}^{\tilde{\tau}_{i+1}} \frac{\|\g_s\|^2}{h^2_{\tau_{i+1}}} \cdot h^2_{\tau_{i+1}} & \leq  h^2_{\tau_{i}}+ \sum_{s=1}^{\tilde{\tau}_{i+2}} \frac{\|\g_s\|^2}{h^2_{s}} \cdot h^2_{\tau_{i+1}}, \\ &   \stackrel{(*)}{\leq} h^2_{\tau_{i}} + \sum_{s=1}^{\tilde \tau_{i+2}} \frac{\|\g_s\|}{h_{s}} \cdot h^2_{\tau_{i+1}}, \\ & \stackrel{\eqref{eq:cond}}{\leq} h^2_{\tau_{i}}+ h_{\tau_{i+1}}  h_{\tau_{i+2}} \leq 2 h_{\tau_{i+2}}^2, \label{eq:vbound}
	\end{align}
	where the inequality ${(*)}$ follows by Assumption \ref{assum:assum2}. Now by \eqref{eq:doubling}, we also have 
	\begin{align}
	\sum_{i=1}^{k-2} h_{\tau_{i+1}} \leq \left(\sum_{i=1}^k\frac{1}{2^i} \right) h_{\tau_{k}}\leq h_{\tau_k}. \label{eq:Lbound}
	\end{align}
	Thus, substituting \eqref{eq:Lbound} and \eqref{eq:vbound} into \eqref{eq:epochs2}, and using the fact that $h_{\tau_{k}}\leq h_T$, we get:
	\begin{align}
	\sum_{s=1}^{\tilde\tau_{k-1}} \inner{\g_s}{\what\w_s-\w} & \leq 4 \|\w\| h_{T} \sqrt{2\ln_+\left(\|\w\|b_T \right)} +  h_T \cdot ( 4 \|\w\| \ln (2 \|\w\| b_T) + 1),\\
	& \leq  h_T \cdot ( 8 \|\w\| \ln_+(2 \|\w\| b_T) + 2\|\w\| +1), \label{eq:pre}
	\end{align}
	where in the last inequality, we used the fact that $\sqrt{2 x}\leq x+1/2$, for all $x\geq 0$. Now, summing \eqref{eq:epochs} over the last two epochs, yields
	\begin{align}
\sum_{s=\tau_{k-1}}^{T} \inner{\g_s}{\what\w_s-\w} &\leq 2 \|\w\| \sqrt{2 V_{T} \ln_+\left(\|\w\|b_T \right)} +  2 h_{T} \cdot(4 \|\w\| \ln_+ (2 \|\w\| b_T) + 1). \label{eq:post}
\end{align}
Adding \eqref{eq:pre} and \eqref{eq:post} together leads to 
\begin{align}
\sum_{s=1}^{T} \inner{\g_s}{\what\w_s-\w} &\leq 2 \|\w\| \sqrt{2 V_{T} \ln_+\left(\|\w\|b_T \right)} +  (16 \|\w\| \ln_+ (2 \|\w\| b_T)+2\|\w\| + 3 ) h_T. \label{eq:final}
\end{align}
This concludes the proof.
	\end{proof}

	\section{Proofs for Section~\ref{sec:multidim}}\label{sec:pf.multidim}

In this section we work on a version of the potential function that does not have the tuning for Section~\ref{sec:multidim} substituted in yet, so that we can prove the result necessary for Section~\ref{sec:linear} in one go. The potential is parameterized by a \emph{prior variance} $\sigma^2 > 0$, \emph{initial wealth} $\epsilon > 0$ and, as before \emph{prod factor} $\gamma > 1$. It is defined by
        \begin{align}
  \Psi(\G, \bm V, h)
  \coloneqq
   \frac{\epsilon \exp \del*{
      \inf_{\lambda \ge 0} \set*{
        \frac{1}{2} \G^\top \del*{\sigma^{-2} \bm I + \gamma \bm V + \lambda \bm I}^{-1} \G
        + \frac{\lambda \rho(\gamma)^2}{2 h^2}
      }
    }
  }{
    \sqrt{\det\del*{\bm I + \sigma^2 \gamma \bm V}}
  }
  ,
  \label{eq:matrixpotential.real}
\end{align}

\subsection{Proof of Lemma~\ref{lemma:multidim.control}}\label{sec:multidim.control.real}

We prove the claim in Lemma~\ref{lemma:multidim.control} for the more general potential \eqref{eq:matrixpotential.real}.
	Let $\lambda_\star \ge 0$ be the minimizer in the problem $\Psi(\G_{t-1}, \bm V_{t-1}, h_t)$. With that notation, we see that $\what\w_t = - \Psi(\G_{t-1}, \bm V_{t-1}, h_t) \cdot
	\del*{\sigma^{-2} \bm I + \gamma \bm V + \lambda_\star \bm I}^{-1} \G_{t-1}$. To prove the lemma, it suffices to prove the stronger statement obtained by picking the sub-optimal choice $\lambda=\lambda_\star$ for the problem $\Psi(\G_t, \bm V_t, h_t)$, and dividing by $\Psi(\G_{t-1}, \bm V_{t-1}, h_t) > 0$, \emph{i.e.}\
	\begin{align*}
	&- \g_t\cdot  \del*{\sigma^{-2} \bm I + \gamma \bm V_{t-1} + \lambda_\star \bm I}^{-1} \G_{t-1}
	\\
	&~\le~
	1 - \frac{\exp \del*{
			\frac{1}{2} \G_t^\top \del*{\sigma^{-2} \bm I + \gamma \bm V_t + \lambda_\star \bm I}^{-1} \G_t
			+ \frac{\lambda_\star \rho(\gamma)^2}{2 h_t^2}
			- \frac{1}{2} \ln \det\del*{\bm I + \sigma^2 \gamma \bm V_t}
	}}{\exp \del*{
			\frac{1}{2} \G_{t-1}^\top \del*{\sigma^{-2} \bm I + \gamma \bm V_{t-1} + \lambda_\star \bm I}^{-1} \G_{t-1}
			+ \frac{\lambda_\star \rho(\gamma)^2}{2 h_t^2}
			- \frac{1}{2} \ln \det\del*{\bm I + \sigma^2 \gamma \bm V_{t-1}}
	}}.
	\end{align*}
	Let us abbreviate $\bm \Sigma^{-1} = \sigma^{-2} \bm I + \gamma \bm V_{t-1} + \lambda_\star \bm I$. The matrix determinant lemma and monotonicity of matrix inverse give
	\[
	\ln \frac{
		\det\del*{\bm I + \sigma^2 \gamma \bm V_t}
	}{
		\det\del*{\bm I + \sigma^2 \gamma \bm V_{t-1}}
	}
	~=~
	\ln \del*{1 + \gamma \g_t^\top \del*{\sigma^{-2} \bm I + \gamma \bm V_{t-1}}^{-1} \g_t}
	~\ge~
	\ln \del*{1 + \gamma \g_t^\top \bm \Sigma \g_t}.
	\]
	Then Sherman-Morrison gives
	\[
	\G_t^\top \del*{\sigma^{-2} \bm I + \gamma \bm V_t + \lambda_\star \bm I}^{-1} \G_t
	~=~
	\G_t^\top \bm \Sigma \G_t
	-
	\gamma
	\frac{
		(\g_t^\top \bm \Sigma \G_t)^2
	}{
		1 + \gamma \g_t^\top \bm \Sigma \g_t
	}
	\]
	and splitting off the last round $\G_t = \G_{t-1} + \g_t$ gives
	\[
	\G_t^\top \del*{\sigma^{-2} \bm I + \gamma \bm V_t + \lambda_\star \bm I}^{-1} \G_t
	~=~
	\G_{t-1}^\top \bm \Sigma \G_{t-1}
	+
	\frac{
		2 \G_{t-1}^\top \bm \Sigma \g_t
		+ \g_t^\top \bm \Sigma \g_t
		- \gamma (\g_t^\top \bm \Sigma \G_{t-1})^2
	}{
		1 + \gamma \g_t^\top \bm \Sigma \g_t
	}
	.
	\]
	All in all, it suffices to show
	\[
	- \g_t^\top \bm \Sigma \G_{t-1}
	~\le~
	1 - \exp \del*{
		\frac{
			2 \G_{t-1}^\top \bm \Sigma \g_t
			+ \g_t^\top \bm \Sigma \g_t
			- \gamma (\g_t^\top \bm \Sigma \G_{t-1})^2
		}{
			2(1 + \gamma \g_t^\top \bm \Sigma \g_t)
		}
		- \frac{1}{2} \ln \del*{1 + \gamma \g_t^\top \bm \Sigma \g_t}
	}.
	\]
	Introducing scalars $r = \g_t^\top \bm \Sigma \G_{t-1}$ and $z = \g_t^\top \bm \Sigma \g_t$, this simplifies to
	\[
	- r
	~\le~
	1 - \exp \del*{
		\frac{
			2 r
			+ z
			- \gamma r^2
		}{
			2(1 + \gamma z)
		}
		- \frac{1}{2} \ln \del*{1 + \gamma z}
	}
	\]
	Being a square, $z \ge 0$ is positive. In addition, optimality of $\lambda_\star$ ensures that $\norm{\bm \Sigma \G_{t-1}} = \frac{\rho(\gamma)}{h_t}$; this follows from the fact that $\frac{\mathrm d}{\mathrm d \lambda}\left. \G_{t-1}^\top (\sigma^{-2} \bm I + \gamma \bm V+\lambda \bm{I})^{-1} \G_{t-1}\right|_{\lambda = \lambda_\star} = \|\bm  \Sigma \G_{t-1}\|^2$. In combination with $\norm{\g_t} \le h_t$, we find $\abs{r} \le \rho(\gamma) \le 1$. The above requirement may hence be further reorganized to
	\[
	2 r
	- \gamma r^2
	~\le~
	- z
	+ (1 + \gamma z) \del*{
		\ln \del*{1 + \gamma z}
		+ 2 \ln (1+r)
	}.
	\]
	The convex right hand side is minimized subject to $z \ge 0$ at
	\[
	z
	~=~
	\max \set*{0,
		\frac{
			e^{
				\frac{1}{\gamma} - 1
				- 2 \ln (1+r)
			} - 1
		}{
			\gamma
		}
	}
	\]
	so it remains to show
	\[
	2 r
	- \gamma r^2
	~\le~
	\begin{cases}
	\frac{1}{\gamma}
	- (1+r)^{-2} e^{
		\frac{1}{\gamma} - 1
	},
	& \text{if}\ \frac{1}{\gamma} - 1
	\ge 2 \ln (1+r);
	\\
	2 \ln (1+r),
	& \text{otherwise.}
	\end{cases}
	\]
	The function $\rho$ in \eqref{def.xi} is designed to satisfy the hardest case, where $r = -\rho(\gamma)$, with equality.

	\subsection{Proof of Theorem~\ref{thm:multidim}}

We restate the claim for the potential~\eqref{eq:matrixpotential.real} before tuning:
\begin{theorem}[Theorem~\ref{thm:multidim} rephrased]
  \label{thm:rephrased}
  Let $\bm \Sigma^{-1}_T \df \sigma^{-2} \bm I + \gamma \bm V_T$. For $(\what\w_t)$ as in \eqref{eq:FTLR.multid}, we have
  \[
    \sum_{t=1}^T \tuple*{\what\w_t - \w, \g_t}
    ~\le~
  \epsilon
  +
  \sqrt{Q_T^\w \ln_+\del*{\frac{
      \det\del*{\sigma^2 \bm \Sigma^{-1}_T}
    }{
      \epsilon^2
    } Q_T^\w}},\quad \text{for all $\w\in \reals^d$, where}
\]
\[
  Q_T^\w
 \coloneqq
  \max \set*{
    \w^\top \bm \Sigma^{-1}_T \w
    ,
    \frac{1}{2} \del*{\frac{h_T^2 \norm{\w}^2}{\rho(\gamma)^2} \ln \del*{
        \frac{
          \det\del*{\sigma^2 \bm \Sigma^{-1}_T}
        }{
          \epsilon^2
        }
        \frac{h_T^2 \norm{\w}^2}{\rho(\gamma)^2}
      }
      + \w^\top \bm \Sigma^{-1}_T \w}
  }
  .
\]
\end{theorem}

	Using that $\Psi(\G, \bm V, h)$ is decreasing in $h$, we can telescope to obtain
	\[
	\sum_{t=1}^T \g_t^\top \what\w_t
	~\le~
	\Psi(\vzero, \vzero, h_1)
	- \Psi(\G_T, \bm V_T, h_T)
	\]
	Using the definition reveals $\Psi(\vzero, \vzero, h_1) = \epsilon$, yielding
	\begin{align}
	\sum_{t=1}^T \g_t^\top \what\w_t
	~\le~
	\epsilon
	- \frac{
		\epsilon
		\exp \del*{
			\inf_{\lambda \ge 0}~
			\frac{1}{2} \G_T^\top \del*{\bm \Sigma^{-1}_T + \lambda \bm I}^{-1} \G_T
			+ \frac{\lambda \rho(\gamma)^2}{2 h_T^2}
		}
	}{
		\sqrt{\det\del*{\sigma^2 \bm \Sigma^{-1}_T}
		}
	}
	. \label{eq:predual}
	\end{align}
	To transform this into a regret bound, it remains to compute the convex conjugate of the RHS of \eqref{eq:predual} in $\G_T$. To this end, let
	\[
	f(\G)
	~=~
	\exp \del*{
		\inf_{\lambda \ge 0}~
		\frac{1}{2} \G^\top \del*{\bm Q + \lambda \bm I}^{-1} \G
		+ \frac{\lambda Z}{2}
	}.
	\]
	The Fenchel dual of this function is
	\begin{align*}
	f^\star(\bm u)
	&~=~
	\sup_{\G}~
	\w^\top \G
	- \exp \del*{
		\inf_{\lambda \ge 0}~
		\frac{1}{2} \G^\top \del*{\bm Q + \lambda \bm I}^{-1} \G
		+ \frac{\lambda Z}{2}
	}
	\\
	&~=~
	\sup_{\G, \lambda \ge 0}~
	\w^\top \G
	- \exp \del*{
		\frac{1}{2} \G^\top \del*{\bm Q + \lambda \bm I}^{-1} \G
		+ \frac{\lambda Z}{2}
	}
	\\
	&~=~
	\sup_{\alpha, \lambda \ge 0}~
	\alpha \w^\top \del*{\bm Q + \lambda \bm I} \w
	- \exp \del*{
		\frac{\alpha^2}{2} \w^\top \del*{\bm Q + \lambda \bm I} \w
		+ \frac{\lambda Z}{2}
	}
	\\
	&~=~
	\sup_{\lambda \ge 0}~
	\sqrt{\w^\top \del*{\bm Q + \lambda \bm I} \w}
	X\del*{\w^\top \del*{\bm Q + \lambda \bm I} \w e^{-\lambda Z}},
	\end{align*}
	where the model complexity is measured for $\theta \ge 0$ through the function $X(\theta) \df
	\sup_{\alpha}~
	\alpha
	-  e^{
		\frac{\alpha^2}{2}
		- \frac{1}{2} \ln \theta
	}.$
	One can write $X(\theta) = W\left(\theta\right)^{1/2} - W\left(\theta\right)^{-1/2}$ in terms of the Lambert function $W$ (where $W(x)$ is defined as the principal solution to $W(x) e^{W(x)} = x$). We will further use that $X(\theta)$ is increasing, and that it satisfies $X(\theta) \le \sqrt{\ln_+ \theta}$ (see Lemma~\ref{eq:Xfn}). Zero derivative of the above objective for $\lambda$ occurs at the pleasantly explicit
	\[
	\lambda
	~=~
	\frac{\ln \frac{\norm{\w}^2}{Z}}{2 Z}
	- \frac{\w^\top \bm Q \w}{2 \norm{\w}^2}
	,
	\]
	and hence the optimum for $\lambda$ is either at that point or at zero, whichever is higher, with the crossover point at $\frac{\norm{\w}^2}{Z} \ln \frac{\norm{\w}^2}{Z} = \w^\top \bm Q \w$.
	Plugging that in, we find that
	\[
f^\star(\w) =	\begin{cases}
	\sqrt{\frac{1}{2} \del*{C
			+ \w^\top \bm Q \w}} X\del*{\frac{1}{2} \del*{
			C
			+ \w^\top \bm Q \w} e^{
			- \frac{\ln \frac{\norm{\w}^2}{Z}}{2}
			+ \frac{Z \w^\top \bm Q \w}{2 \norm{\w}^2}
	}},
	&
\text{if}\	C
	\ge \w^\top \bm Q \w;
	\\
	\sqrt{\w^\top \bm Q \w} X(\w^\top \bm Q \w),
	&
	\text{otherwise,}
	\end{cases}
	\]
	where $C \df \frac{\norm{\w}^2}{Z} \ln \frac{\norm{\w}^2}{Z}$.
	Using that $X(\theta)$ is increasing, we may drop the exponential in its argument in the first case, and obtain
	\[
	f^\star(\w)
	~\le~
	\sqrt{Q_T^\w} X(Q_T^\w)
	\quad
	\text{where}
	\quad
	Q_T^\w
	~\df~
	\max \set*{
		\w^\top \bm Q \w
		,
		\frac{1}{2} \del*{\frac{\norm{\w}^2}{Z} \ln \frac{\norm{\w}^2}{Z}
			+ \w^\top \bm Q \w}
	}
	.
	\]
	Note that this is a curious maximum between $\w^\top \bm Q \w$ (the larger for modest $\w$), and the \emph{average} between that very same term and another quantity that grows super-linearly with $\norm{\w}^2$ (so this is the winner for extreme $\w$).
	
	Okay, now let's collect everything for the final result and undo the abbreviations. We have
	\begin{align*}
	\sum_{t=1}^T \g_t^\top \what\w_t
	&~\le~
	\epsilon
	+  \inf_\w~
	\w^\top \G_T
	+
	\frac{
		\epsilon
	}{
		\sqrt{\det\del*{\sigma^2 \bm \Sigma^{-1}_T}
		}
	}
	f^\star\del*{-\frac{
			\sqrt{\det\del*{\sigma^2 \bm \Sigma^{-1}_T}
			}
		}{
			\epsilon
		} \w},
	\\
	&~\le~
	\epsilon
	+  \inf_\w~
	\w^\top \G_T
	+
	\sqrt{Q_T^\w} X\del*{\frac{
			\det\del*{\sigma^2 \bm \Sigma^{-1}_T}
		}{
			\epsilon^2
		} Q_T^\w},
	\end{align*}
	where\[
	Q_T^\w
	\coloneqq 
	\max \set*{
		\w^\top \bm \Sigma^{-1}_T \w
		,
		\frac{1}{2} \del*{\frac{h_T^2 \norm{\w}^2}{\rho(\gamma)^2} \ln \del*{
				\frac{
					\det\del*{\sigma^2 \bm \Sigma^{-1}_T}
				}{
					\epsilon^2
				}
				\frac{h^2_T \norm{\w}^2}{\rho(\gamma)^2}
			}
			+ \w^\top \bm \Sigma^{-1}_T \w}
	}
	.
	\]
	To complete the proof of Theorem~\ref{thm:rephrased}, it remains to prove the following result.
	\begin{lemma}\label{eq:Xfn}
		For $\theta \ge 0$, define $
		X(\theta)
		\coloneqq
		\sup_{\alpha}~
		\alpha
		-  e^{
			\frac{\alpha^2}{2}
			- \frac{1}{2} \ln \theta
		}$. Then $X(\theta)=(W(\theta))^{1/2} - (W(\theta))^{-1/2} = \sqrt{\ln \theta} + o(1)$.
	\end{lemma}
	
	\begin{proof}
	The fact that $X(\theta)=(W(\theta))^{1/2} - (W(\theta))^{-1/2}$ follows from \cite[Lemma 18]{orabona2016coin}. Recall that
		\[
		\sup_x~ y x - e^x
		~=~
		y \ln y - y
		\]
		Hence
		\begin{align*}
		X(\theta)
		&~=~
		\sup_{\alpha}~
		\alpha
		-  e^{
			\frac{\alpha^2}{2}
			- \frac{1}{2} \ln \theta
		}
		\\
		&~=~
		\sup_{\alpha}
		\inf_\eta
		~
		\alpha
		- \eta \del*{
			\frac{\alpha^2}{2}
			- \frac{1}{2} \ln \theta
		}
		+ \eta \ln \eta
		- \eta
		\\
		&~=~
		\inf_\eta
		~
		\frac{1}{2 \eta}
		+ \frac{\eta}{2} \ln \theta
		+ \eta \ln \eta
		- \eta
		\\
		&~\le~
		\min \set*{
			\sqrt{\ln \theta}
			- \frac{1 + \frac{1}{2} \ln \ln \theta}{\sqrt{\ln \theta}},
			\frac{\sqrt{\theta}}{2}
			- \frac{1}{\sqrt{\theta}}
		}
		\\
		&~\le~
		\sqrt{\ln_+ \theta}
		\end{align*}
		where we plugged in the sub-optimal choices
		$\eta
		=
		\frac{1}{\sqrt{\ln \theta}}$ (this requires $\theta \ge 1$)
		and $
		\eta
		=
		\frac{1}{\sqrt{\theta}}$.
		When we stick in $\eta = \frac{1}{\sqrt{\ln (e^{e^{-2}}+\theta)}}$ we find
		\[
		X(\theta)
		~\le~
		\frac{
			\ln (e^{e^{-2}}+\theta)
			+ \ln \theta
			- \ln \del*{\ln (e^{e^{-2}}+\theta)}
			- 2
		}{
			2 \sqrt{\ln (e^{e^{-2}}+\theta)}
		}
		~\le~
		\sqrt{\ln (e^{e^{-2}}+\theta)}
		\]
		Note that $e^{e^{-2}} = 1.14492$. This is less than $2$, the value of $\theta$ where $\sqrt{\theta}/2-1/\sqrt{\theta}$ becomes positive.
	\end{proof}

\section{Proofs for Section~\ref{sec:lower}}
\label{sec:lowerproof}

\subsection{Proof of Lemma \ref{lem:seconbound}}
Let $c,b,\beta\geq 0$, $\nu\geq 1$, $\alpha\in ]1,2]$, and $\gamma \in]-1,-\alpha^{-1}[$. We consider the $1$-dimensional case (\emph{i.e.} $d=1$) and set $g_t = t^{\gamma}$, for all $t\geq 1$. Since $-1<  \gamma < -1/\alpha$, we have $L_t=L_1=1$, for all $t\geq 1$, and so the sequence $(\sqrt{V_{\alpha, t} \ln (t)/L^{\alpha}_t})$ is increasing. Further, there exists $p,q>0$ such that,
\begin{align}
\forall  t\geq 1, \quad p   \sqrt{\ln t} \leq \sqrt{V_{\alpha, t} \ln (t)/L^{\alpha}_t}= \sqrt{\ln t \sum_{s=1}^t s^{\alpha \gamma}} \leq q  \sqrt{\ln t}. \label{eq:sandwitch}
\end{align}
Thus, given any sequence $(\what w_t)\in \reals$ satisfying \begin{align}|\what w_t| \leq b \sqrt{V_{\alpha, t} \ln (t)/L^{\alpha}_t}, \ \ t\geq 1,\end{align} we have, for $T\geq 1$ and $w =- 2  b  \sqrt{V_{\alpha,T} \ln (T)/L^{\alpha}_T}$,
\begin{align}
\sum_{t=1}^T g_t \cdot (\what w_t-w)&\geq  {b}\sqrt{V_{\alpha,T} \ln (T)/L^{\alpha}_T}\cdot \sum_{t=1}^T g_t, \\ & \stackrel{\eqref{eq:sandwitch}}{\geq} b p  \sqrt{\ln T}   \cdot \sum_{t=1}^T t^\gamma, \\
& \geq \frac{b p  \sqrt{\ln T}}{\gamma+1} \cdot   ((T+1)^{\gamma+1}-1). \label{eq:last}
\end{align}
Now by the choice of $w$ and \eqref{eq:sandwitch}, we have $|w| \leq 2 b q \sqrt{\ln T}$, and so by \eqref{eq:last},
\begin{align}
\sum_{t=1}^T g_t \cdot  (\what w_t-w) \geq 	L_T |w|^{\nu} \cdot  \frac{ p\cdot ((T+1)^{\gamma+1}-1)}{(\gamma +1)(2q)^\nu b^{\nu-1} (\ln T)^{\nu/2-1/2}}   . \label{eq:firstrequire}
\end{align}
Using again the fact that $|w| \leq 2 b q \sqrt{\ln T}$ and \eqref{eq:sandwitch}, we have $L_T^{1-\alpha/2} (|w| +1) \sqrt{V_{\alpha,T} \ln T}\leq  2b q^2 \ln T  + q \sqrt{\ln T}$, and so due to \eqref{eq:last}, we have
\begin{align}
\sum_{t=1}^T g_t \cdot  (\what w_t-w) \geq L_T^{1-\alpha/2} (|w| +1) \sqrt{V_T \ln T}  \cdot   \frac{(T+1)^{\gamma+1}-1}{(\gamma+1)\left(\frac{2 q^2}{p}\sqrt{\ln T} + \frac{q}{b p}\right)}. \label{eq:secondrequire}
\end{align}
Since $ \gamma>-1$, the exists $T\geq 1$ such that
\begin{align}
2 c \cdot  \ln (1+|w| T)^{\beta} \leq  \min \left( \frac{(T+1)^{\gamma+1}-1}{(\gamma+1)\left(\frac{2 q^2}{p}\sqrt{\ln T} + \frac{q}{b p}\right)}, \ \  \frac{ p\cdot ((T+1)^{\gamma+1}-1)}{(\gamma +1)(2q)^\nu b^{\nu-1} (\ln T)^{\nu/2-1/2}}  \right),
\end{align}
and so for such a choice of $T$, \eqref{eq:firstrequire} and \eqref{eq:secondrequire} imply the desired result.

\subsection{Proof of Theorem \ref{thm:lower2}}
\label{sec:proofoflower2}
We need the following lemma in the proof of Theorem \ref{thm:lower2}:
\begin{lemma}
	\label{lem:secondbound3}
	For all $b,c,\beta\geq 0$ and $\nu \in[1,3[$, there exists $(\g_t)\in\reals^d$, $T\geq 1$, and $\w\in \reals^d$, such that for any sequence $(\what\w_t)$ satisfying $\|\what\w_t\| \leq b \cdot \sqrt{t  \ln t}$, for all $t\geq 1$, we have\begin{gather}
	\sum_{t=1}^T \inner{\what\w_t-\w}{\g_t} \geq  c \cdot \ln (1+\|\w\| T)^{\beta}\cdot (L_T \|\w\|^\nu +  L_T (\|\w\|+1) \sqrt{T \ln T}).
	\end{gather}
\end{lemma}
\begin{proof}
	Let $c,b,\beta,\geq 0$, $\nu\in[1,3[$, and $\alpha \in]1,2]$. We consider the $1$-dimensional case (\emph{i.e.} $d=1$) and set $g_t = 1$, for all $t\geq 1$. In this case, we have $L_t=1$, for all $t\geq 1$.
	Given any sequence $(\what w_t)\in \reals$ satisfying \begin{align}|\what w_t| \leq  b \sqrt{t \ln t} , \ \ t\geq 1, \label{eq:cons} \end{align} we have, for $T\geq 1$ and $w =-2 b \sqrt{T \ln T} $,
	\begin{align}
	\sum_{t=1}^T g_t \cdot (\what w_t-w)& \stackrel{\eqref{eq:cons}}{\geq} b \sqrt{T \ln T}\cdot \sum_{t=1}^T g_t, \\ & =b  \sqrt{T\ln T}   \cdot T. \label{eq:last2}
	\end{align}
	Now since $|w| = 2 b  \sqrt{T\ln T}$, we have, by \eqref{eq:last2},
	\begin{align}
	\sum_{t=1}^T g_t \cdot  (\what w_t-w) \geq 	L_T |w|^{\nu} \cdot  \frac{ T^{3/2-\nu/2}}{2^{\nu} b^{\nu-1} (\ln T)^{\nu/2-1/2}  } . \label{eq:firstrequire2}
	\end{align}
	Using again the fact that $|w| = 2 b  \sqrt{ T\ln T}$ and $L_T=1$, we have $L_T (|w| +1) \sqrt{T\ln T}=  2b T \ln T  +  \sqrt{T \ln T}$, and so due to \eqref{eq:last2},
	\begin{align}
	\sum_{t=1}^T g_t \cdot  (\what w_t-w) \geq L_T (|w| +1) \sqrt{T \ln T}  \cdot   \frac{T}{2\sqrt{T \ln T} +1/b}. \label{eq:secondrequire2}
	\end{align}
	Since $ \nu\in[1,3[$, the exists $T\geq 1$ such that
	\begin{align}
2	c \cdot  \ln (1+|w|T)^{\beta} \leq  \min \left(  \frac{T}{2\sqrt{T \ln T} +1/b}, \ \   \frac{ T^{3/2-\nu/2}}{2^{\nu} b^{\nu-1} (\ln T)^{\nu/2-1/2}  }   \right),
	\end{align}
	and so for such a choice of $T$, \eqref{eq:firstrequire2} and \eqref{eq:secondrequire2} imply the desired result.
\end{proof}

\begin{proof}{\textbf{of Theorem \ref{thm:lower2}.}}
	By Lemma \ref{lem:firstbound}, the only candidate algorithms are the ones whose outputs $(\what \w_t)$ satisfy $\|\what\w_t\| \leq b \sqrt{t \ln t}$, for all $t\geq 1$, for some constant $b>0$. By Lemma \ref{lem:secondbound3}, no such algorithms can achieve the desired regret bound.
\end{proof}

\section{Proof of Section \ref{sec:linear}}
\label{sec:sec5proofs}
\begin{proof}{\textbf{of Theorem \ref{thm:freegradol}.}}
	The proof is similar to that of Theorem \ref{thm:freegrad1} expect for some changes to account for the fact that the modified \freerange\ wrapper scales the outputs of \freegrad.
	
	First, let us review some notation. Let $k\geq 1$ be the total number of epochs and denote by $\tau_i\geq 1$ the start index of epoch $i\in[k]$. Further, for $\tau,\tau' \in \mathbb{N}$, we define $\tilde\tau \coloneqq \tau -1$, $V_{\tau:\tau'}\coloneqq |x_\tau|^2+ \sum_{s=\tau}^{\tilde\tau'} |g_s|^2$, and $B_\tau \coloneqq \sum_{s=1}^\tau |x_s|/h_s$. In what follows, let $w\in \reals$ be fixed. 
	
	Let $(\what u_t)$ be the outputs of algorithm \ref{alg:newfreerange} and $i\in[k]$. In this case, we have $\what u_t = \what w_t/(h_{\tau_i} B_{\tau_i})$, for all $t\in\{\tau_i,\dots, \tilde{\tau}_{i+1} \}$, and by Theorem \ref{thm:scaled}, with $d=1$, $\epsilon = h_{\tau_i} B_{\tau_i}$, and $g_t \in x_t \cdot \partial^{(0,1)} \ell(y_t, x_t \what u_t)$: 
	 \begin{align}
	 \sum_{t=\tau_i}^{\tilde{\tau}_{i+1}} g_t \cdot (\what u_t-w) & \leq  2 |w| \sqrt{V_{\tau_i:\tau_{i+1}}\ln_+ \left(\frac{ 2 h_{\tau_i} B_{\tau_i}  |w| V_{\tau_i:\tau_{i+1}}  }{h^2_{\tau_i}} \right)} \\ & \quad  +  4 h_{\tau_{i+1}} |w|  \ln \left(  \frac{4 h_{\tilde \tau_{i+1}} h_{\tau_i} B_{\tau_i}  |w| \sqrt{V_{\tau_i:\tau_{i+1}}} }{ h_{\tau_i}^2} \right)  + \frac{h_{\tau_i} }{h_{\tau_i} B_{\tau_i}},\\
	 & \leq  2 |w| \sqrt{V_{\tau_i:\tau_{i+1}}\ln_+ \left(\frac{ 2B_{\tau_i}  |w| V_{\tau_i:\tau_{i+1}}  }{h_{\tau_i}} \right)} \\ & \quad  +  4 h_{\tau_{i+1}} |w|  \ln \left(  \frac{4 h_{\tilde{\tau}_{i+1}}  B_{\tau_i}  |w| \sqrt{V_{\tau_i:\tau_{i+1}}}  }{ h_{\tau_i}} \right)  + \frac{1}{B_{\tau_i}},  \label{eq:newregre}
	 \end{align}
	Recall that at epoch $i\in[k]$, the restart condition in Algorithm \ref{alg:newfreerange} is triggered at $t=\tau_{i+1}\geq \tau_{i}$ only if 
	\begin{align}
	\frac{h_t}{h_{\tau_i}}\  \stackrel{(*)}{>} \ \sum_{s=1}^{t-1} \frac{|x_s|}{h_s} +1= \sum_{s=1}^t \frac{|x_s|}{h_s}, \label{eq:cond0}
	\end{align}
	where the equality follows by the fact that when $(*)$ is satisfied for the first time, it must hold that $|x_t|=h_t$ (recall that the hints $(h_t)$ satisfy \eqref{eq:lin2}); in fact, we have,
	\begin{align}
	h_{\tau_{i}} = |x_{\tau_{i}}|, \quad  \text{for all } i\in[k]. \label{eq:equalitites}
	\end{align}
	From \eqref{eq:cond0}, we get that
	\begin{align}
	\frac{h_{\tilde\tau_{i+1}}}{h_{\tau_{i}}} \leq  \frac{\sqrt{V_{\tau_{i}:\tau_{i+1}}}}{h_{\tau_i}} \leq  \sqrt{\sum_{t=\tau_{i}}^{\tilde{\tau}_{i+1}} \left(\sum_{s=1}^{t} \frac{|x_s|}{h_s}\right)^2}  \leq \sqrt{\frac{b_T}{2}}, 	 \label{eq:cons10}
	\end{align}
	where $b_T \coloneqq 2 \sum_{t=1}^T (\sum_{s=1}^t {|x_s|}/{h_s})^2$. Plugging \eqref{eq:cons10} into \eqref{eq:newregre}, and letting $c_T\coloneqq  B_T^2 b_T$, we get:
	\begin{align}
	\sum_{t=\tau_i}^{\tilde{\tau}_{i+1}} g_t \cdot (\what u_t-w) & \leq  2 |w| \sqrt{V_{\tau_i:\tau_{i+1}}\ln_+ ( |w| \sqrt{2V_{T} c_T})}  \\ & \quad  +  4 h_{\tau_{i+1}} |w|  \ln \left(  2  |w| \sqrt{2V_{T} c_T }  \right)  + \frac{1}{B_{\tau_i}},\label{eq:epochs0}
	\end{align}
	Summing this inequality over $i=1,\dots, k-2$, we get:
	\begin{align}
		\sum_{t=1}^{\tilde\tau_{k-1}} g_t \cdot (\what u_t-w)&\leq  2 |w|  \sqrt{ k \sum_{i=1}^{k-2} V_{\tau_i:\tau_{i+1}}\ln_+ \left( |w| \sqrt{2V_{T} c_T} \right)} \\ & \quad  +    \sum_{i=1}^{k-2}  4 h_{\tau_{i+1}} |w|  \ln \left(  2  |w| \sqrt{2V_{T} c_T }  \right)  + \sum_{i=1}^{k-2}   \frac{1}{B_{\tau_i}},  \label{eq:epochspre} , \\ &\leq  2 |w|  \sqrt{ k \cdot  \left(\sum_{s=1}^{\tau_{k-1}} |x_s|^2 +\sum_{i=1}^{k-2} h^2_{\tau_i}\right)\ln_+ \left( |w| \sqrt{2V_{T} c_T} \right)} \\ & \quad  +    \sum_{i=1}^{k-2}  4 h_{\tau_{i+1}} |w|  \ln \left(  2  |w| \sqrt{2V_{T} c_T }  \right)  + \sum_{i=1}^{k-2}   \frac{1}{B_{\tau_i}},   
		\\ &\stackrel{\eqref{eq:equalitites}}{\leq}  2 |w|  \sqrt{ 2 k   \sum_{s=1}^{\tau_{k-1}} |x_s|^2 \ln_+ \left( |w| \sqrt{2V_{T} c_T} \right)} \\ & \quad  +    \sum_{i=1}^{k-2}  4 h_{\tau_{i+1}} |w|  \ln \left(  2  |w| \sqrt{2V_{T} c_T }  \right)  + \sum_{i=1}^{k-2}   \frac{1}{B_{\tau_i}}.
		 \label{eq:epochs20} 
	\end{align}
	Using \eqref{eq:cond0} again, we get that
	 \begin{align}
\sum_{s=1}^{\tau_{k-1}} |x_s|^2  = \sum_{s=1}^{\tau_{k-1}} \frac{|x_s|^2}{h^2_{\tau_{k-1}}} \cdot h^2_{\tau_{k-1}} \leq  \sum_{s=1}^{\tau_{k}} \frac{|x_s|^2}{h^2_{s}} \cdot h^2_{\tau_{k-1}}  \leq  \sum_{s=1}^{\tau_{k}} \frac{|x_s|}{h_{s}} \cdot h^2_{\tau_{k-1}} & \stackrel{(*)}{\leq}  \left(\sum_{s=1}^{\tau_{k}} \frac{|x_s|}{h_{s}} \right)^2 \cdot \frac{h^2_{\tau_{k-1}}}{k},  \\ & \stackrel{\eqref{eq:cond0}}{\leq}  \frac{h^2_{\tau_{k}}}{k}, \label{eq:vbound0}
	\end{align}
	where the inequality $(*)$ follows by the fact that $\sum_{s=1}^{\tau_{k}} {|x_s|}/{h_{s}}\geq k$ due to \eqref{eq:equalitites}. We also have
	\begin{gather}
	\sum_{i=1}^{k-2} h_{\tau_{i+1}} \leq \sum_{s=1}^{\tau_{k-1}} |x_s|    =\sum_{s=1}^{\tau_{k-1}} \frac{|x_s|}{h_{\tau_{k-1}}} \cdot h_{\tau_{k-1}}  \leq \sum_{s=1}^{\tau_{k}} \frac{|x_s|}{h_{s}} \cdot h_{\tau_{k-1}}  \stackrel{\eqref{eq:cond0}}{\leq}   h_{\tau_{k}}. \label{eq:Lbound0} \\ 
	\sum_{i=1}^{k-2} \frac{1}{B_{\tau_i}} \stackrel{\eqref{eq:equalitites}}{=} \sum_{i=1}^{k-2} \frac{\frac{|x_{\tau_i}|}{h_{\tau_i}}}{B_{\tau_i}} = \sum_{i=1}^{k-2} \frac{\frac{|x_{\tau_i}|}{h_{\tau_i}}}{\sum_{s=1}^{\tau_i} \frac{|x_s|}{h_s}}  \leq \sum_{t=1}^{\tau_{k-2}} \frac{\frac{|x_{t}|}{h_{t}}}{\sum_{s=1}^t \frac{|x_s|}{h_s}}\leq \ln B_T. \label{eq:theone}
	\end{gather}
	Thus, substituting \eqref{eq:theone}, \eqref{eq:Lbound0}, and \eqref{eq:vbound0} into \eqref{eq:epochs20}, and using the fact that $h_{\tau_k}\leq h_T$, we get:
	\begin{align}
	\sum_{t=1}^{\tilde\tau_{k-1}} g_t \cdot (\what u_t-w)  &\leq  2 |w| h_T \sqrt{2\ln_+ ( |w| \sqrt{2V_{T} c_T} )} \\ & \quad  +    4 h_{T} |w|  \ln \left(  2 |w| \sqrt{2V_{T} c_T }  \right)  +   \ln B_T, \\ 
	& \leq    h_{T} |w| \cdot( 6 \ln _+(  2 |w| \sqrt{2V_{T} c_T }  ) + 1  ) +   \ln B_T, \label{eq:pre0}
	\end{align}
	where in the last inequality, we used the fact that $\sqrt{2x}\leq x+1/2$, for all $x\geq 0$. Now, summing \eqref{eq:epochs0} over the last two epochs, yields
	\begin{align}
		\sum_{t=\tau_{k-1}}^{T}  g_t \cdot (\what u_t-w) & \leq  2 |w| \sqrt{2V_{T}\ln_+ ( |w| \sqrt{2V_{T} c_T} )}   +  8 h_{T} |w|  \ln ( 2    |w| \sqrt{2V_{T} c_T }  )   + 2 . \label{eq:post0}
	\end{align}
	Adding \eqref{eq:pre0} and \eqref{eq:post0} together implies the desired result.
	\end{proof}

\end{appendix}


\end{document}